\documentclass[twoside,11pt,abbrvbib,preprint]{article}

\usepackage{jmlr2e}
\usepackage{microtype}
\usepackage{amsmath}
\usepackage{amssymb}
\usepackage{mathtools}
\usepackage{booktabs}
\usepackage{colortbl}
\usepackage{enumitem}
\usepackage{braket}
\usepackage{subfig}
\usepackage{float}
\usepackage{cleveref}
\usepackage{multirow, makecell, xcolor}
\usepackage{threeparttable}

\usepackage{hyperref}
\hypersetup{
    colorlinks=true,
    linkcolor=blue,
    citecolor=black,
    urlcolor=blue,
    filecolor=black,
    pdfborder={0 0 0}
}

\newtheorem{assumption}{Assumption}
\numberwithin{equation}{section}
\numberwithin{theorem}{section}
\numberwithin{assumption}{section}

\usepackage[ruled,vlined,algo2e]{algorithm2e}
\SetAlgoNlRelativeSize{0}
\SetCommentSty{small}
\SetKwComment{tcp}{$\triangleright$~}{}
\SetKw{KwBy}{by}
\SetKw{KwTo}{to}
\SetKw{KwIn}{in}
\SetKw{KwAnd}{and}
\SetKw{KwOr}{or}

\usepackage{tikz}
\usetikzlibrary{shapes,positioning,calc,backgrounds,fit,decorations.pathmorphing}

\newcommand{\BigO}{\mathcal{O}}
\newcommand{\E}{\mathbb{E}}
\newcommand{\R}{\mathbb{R}}
\newcommand{\cG}{\mathcal{G}}
\newcommand{\cV}{\mathcal{V}}
\newcommand{\cE}{\mathcal{E}}
\newcommand{\cF}{\mathcal{F}}
\newcommand{\half}{\tfrac12}
\DeclareMathOperator{\prox}{prox}
\DeclareMathOperator{\dist}{dist}
\DeclareMathOperator*{\argmin}{argmin}

\DeclareMathOperator{\Unif}{Unif}
\newcommand{\tsum}{\mathop{\textstyle\sum}\nolimits}

\crefformat{equation}{(#2#1#3)}
\crefname{subsection}{Section}{sections}
\Crefname{subsection}{Section}{Sections}
\crefname{algocf}{Algorithm}{Algorithms}
\Crefname{algocf}{Algorithm}{Algorithms}
\crefname{lemma}{Lemma}{Lemmas}
\Crefname{lemma}{Lemma}{Lemmas}
\crefname{theorem}{Theorem}{Theorems}
\Crefname{theorem}{Theorem}{Theorems}
\crefname{proposition}{Proposition}{Propositions}
\Crefname{proposition}{Proposition}{Propositions}
\crefname{corollary}{Corollary}{Corollaries}
\Crefname{corollary}{Corollary}{Corollaries}
\crefname{definition}{Definition}{Definitions}
\Crefname{definition}{Definition}{Definitions}
\crefname{remark}{Remark}{Remarks}
\Crefname{remark}{Remark}{Remarks}
\Crefname{assumption}{Assumption}{Assumptions}
\crefname{assumption}{Assumption}{Assumptions}
\crefname{table}{Table}{Tables}
\Crefname{table}{Table}{Tables}
\crefname{figure}{Figure}{Figures}
\Crefname{figure}{Figure}{Figures}

\AddToHook{env/assumption/begin}{\crefalias{theorem}{assumption}}
\AddToHook{env/lemma/begin}{\crefalias{theorem}{lemma}}
\AddToHook{env/proposition/begin}{\crefalias{theorem}{proposition}}
\AddToHook{env/corollary/begin}{\crefalias{theorem}{corollary}}
\AddToHook{env/definition/begin}{\crefalias{theorem}{definition}}

\newcommand{\abs}[1]{\lvert #1 \rvert}

\newcommand{\norm}[1]{\lVert #1 \rVert}
\newcommand{\Norm}[1]{\norm{#1}}

\newcommand{\BlockProx}{\textup{\textsc{BlockProx}}}
\newcommand{\RandomEdge}{\textup{\textsc{RandomEdge}}}

\makeatletter
  \long\def\@makefntext#1{\@setpar{\@@par\@tempdima \hsize
               \advance\@tempdima-15pt\parshape \@ne 15pt \@tempdima}\par
               \parindent 2em\noindent \hbox to \z@{\hss$^{\@thefnmark}$ \hfil}#1}
\makeatother

\ifpdf
\hypersetup{
  pdftitle={Decentralized Optimization with Topology-Independent Communication},
  pdfauthor={Ying Lin, Yao Kuang, Ahmet Alacaoglu, and Michael P. Friedlander}
}
\fi

\ShortHeadings{Topology-Independent Decentralized Optimization}{Lin, Kuang, Alacaoglu, Friedlander}

\begin{document}

\title{Decentralized Optimization with\\ Topology-Independent Communication\thanks{Date: 16 September 2025}}

\author{\name Ying Lin\thanks{Equal contribution.} \email ylin95@hku.hk \\
       \addr Department of Data and Systems Engineering\\
       The University of Hong Kong\\
       Hong Kong
       \AND
       \name Yao Kuang\footnotemark[2] \email yaokuang@cs.ubc.ca \\
       \addr Department of Computer Science\\
       The University of British Columbia\\
       Vancouver, BC, Canada
       \AND
       \name Ahmet Alacaoglu \email alacaoglu@math.ubc.ca \\
       \addr Department of Mathematics\\
       The University of British Columbia\\
       Vancouver, BC, Canada
       \AND
       \name Michael P. Friedlander \email michael.friedlander@ubc.ca \\
       \addr Departments of Computer Science and Mathematics\\
       The University of British Columbia\\
       Vancouver, BC, Canada}

\editor{TBD}
\maketitle

\begin{abstract}
Distributed optimization requires nodes to coordinate, yet full synchronization scales poorly. When $n$ nodes collaborate through $m$ pairwise regularizers, standard methods demand $\BigO(m)$ communications per iteration.
This paper proposes randomized local coordination: each node independently samples one regularizer uniformly and coordinates only with nodes sharing that term. This exploits partial separability, where each regularizer $G_j$ depends on a subset $S_j \subseteq \{1,\ldots,n\}$ of nodes. For graph-guided regularizers where $|S_j|=2$, expected communication drops to exactly 2 messages per iteration.
This method achieves $\tilde{\BigO}(\varepsilon^{-2})$ iterations for convex objectives and under strong convexity, $\BigO(\varepsilon^{-1})$ to an $\varepsilon$-solution and $\BigO(\log(1/\varepsilon))$ to a neighborhood. Replacing the proximal map of the sum $\sum_j G_j$ with the proximal map of a single randomly selected regularizer $G_j$ preserves convergence while eliminating global coordination. Experiments validate both convergence rates and communication efficiency across synthetic and real-world datasets.
\end{abstract}

\begin{keywords}
  decentralized optimization, communication efficiency, partial separability,
  proximal methods, multi-task learning
\end{keywords}

\section{Introduction}

We study distributed optimization with partially separable regularizers, where each term depends on overlapping subsets of variables. This structure appears, for example, when hospitals collaborate on diagnostic models without centralizing patient data~\citep{CSPEC2021}, or when sensor networks learn environmental patterns through local interactions~\citep{PKP2006,KMR2012,LZT2015}. In these cases, the communication costs of coordinating between nodes dominates local computation. 

Standard proximal methods require all nodes to synchronize at each iteration, which incurs a communication cost proportional to the number of pairwise interactions. In large-scale distributed systems, such global coordination becomes prohibitive. We show that randomized local coordination achieves the convergence rates of centralized proximal methods.

Our approach exploits the structure of partially separable regularizers with overlapping variable dependencies. Each node independently samples one of $m$ regularizers uniformly at random and coordinates only with nodes sharing that regularizer. This reduces expected communication from $\BigO(m)$ to $\sum_{j=1}^m (a_j^2 - a_j)/m$ messages per iteration, where $a_j$ counts the nodes invoved in component $j$. Communication thus depends only on local collaboration sizes, not global network topology, while preserving optimal convergence rates.

We formalize our analysis through a structured optimization framework
\begin{equation}
\label{eq:general_problem}
\min_{x\in\mathbb{R}^{nd}}\ \left\{ H(x)\coloneqq F(x) + G(x)\right\},
\end{equation}
where the objective exhibits two distinct forms of separability:
\begin{equation}
    \label{eq:decomposed_problem}
    \textstyle
    F(x) \coloneqq \sum_{i=1}^n f_i(x_i)
    \quad\text{and}\quad
    G(x) \coloneqq \sum_{j=1}^m G_j(x).
\end{equation}
The variable $x=(x_1, \ldots, x_n)\in\R^{nd}$ partitions into blocks $x_i\in\R^d$, where each $f_i: \R^d \to \R$ and $G_j: \R^{nd} \to \R$ is proper, closed, and convex. Each block $x_i$ corresponds to node~$i$ in the network, where $f_i$ represents the local data and objective for node $i$, and $G_j$ represents the $j$th coordination constraint.

To see why this structure matters, again consider the hospital collaboration scenario. Each hospital $i$ maintains its own diagnostic model parameters $x_i$, learning from local patient data through its individual loss function $f_i(x_i)$. The separable structure of $F$ in~\eqref{eq:decomposed_problem} allows each hospital to update its model independently, preserving privacy.

Hospitals, however, in similar regions or similar patient populations should develop similar diagnostic criteria. We might capture these relationships through a \emph{graph-guided regularizer} that connects hospitals based on geographical proximity, shared protocols, or demographic similarity. If hospitals $i$ and $j$ should coordinate, we include a regularization term $g_{ij}(x_i, x_j)$ that encourages their models to remain similar.

\Cref{fig:partial_separability_hypergraph} illustrates this structure for a network of thirteen hospitals with various collaboration relationships. Orange edges show pairwise collaborations; shaded regions indicate larger consortia. This model allows for multiple levels of collaboration. Each regularizer component $G_j$ couples only the hospitals participating in that specific relationship. Hospital networks exhibit \emph{partial separability}.

\begin{figure}[t]
    \centering
    \usetikzlibrary{shapes,positioning,calc,backgrounds,fit,decorations.pathmorphing}

\begin{tikzpicture}[scale=1.3]
    \tikzset{hyperedge thickness/.style={line width=2.0pt}}
    
    \tikzset{hyperedge border/.style={rounded corners=8pt,fill=none,hyperedge thickness}}
    
    \tikzstyle{node}=[circle,draw,thick,minimum size=0.6cm,inner sep=2pt,fill=white]
    \tikzstyle{hyperedge}=[rounded corners=8pt,opacity=0.3]
    \tikzstyle{comm}=[thick,->,>=stealth,dashed,decorate,decoration={amplitude=1.5pt,segment length=8pt}]
    \tikzstyle{sample}=[thick,->,>=stealth,ultra thick,decorate,decoration={amplitude=1.5pt,segment length=8pt}]
    
    \node[node] (n1) at (0.8,0.8) {1};
    \node[node] (n2) at (3,0.3) {2};  %
    \node[node] (n3) at (3.5,0.8) {3};
    \node[node] (n4) at (5.2,2.1) {4};  %
    \node[node] (n5) at (2.5,4.5) {5};  %
    \node[node] (n6) at (0.5,4) {6};    %
    \node[node] (n7) at (-0.5,2.0) {7}; %
    \node[node] (n8) at (0.3,1.6) {8};  %
    \node[node] (n9) at (1.5,2) {9};    %
    \node[node] (n10) at (3,2.6) {10};  %
    \node[node] (n11) at (4.2,2.3) {11};  %
    \node[node] (n12) at (2,3.7) {12};  %
    \node[node] (n13) at (4.2,0.3) {13}; %

    \draw[thick,orange!80,line width=2.5pt] (n1) -- (n2) node[pos=.6,below,font=\small] {$G_1$};
    \draw[thick,orange!80,line width=2.5pt] (n7) -- (n8) node[pos=1.5,left=7pt,font=\small] {$G_2$};
    \draw[thick,orange!80,line width=2.5pt] (n11) -- (n13) node[pos=0.67,right,font=\small] {$G_3$};
    
    \begin{scope}[on background layer]
        \node[fit=(n2)(n3)(n13),hyperedge,fill=cyan!40,inner sep=10pt] (G4region) {};
        \node[fit=(n2)(n3)(n13),hyperedge border,draw=cyan!70,inner sep=10pt] (G4border) {};
        
        \node[fit=(n5)(n6)(n12),hyperedge,fill=blue!40,inner sep=8pt] (G5region) {};
        \node[fit=(n5)(n6)(n12),hyperedge border,draw=blue!70,inner sep=8pt] {};
        
        \node[fit=(n9)(n10)(n12),hyperedge,fill=brown!40,inner sep=8pt] (G6region) {};
        \node[fit=(n9)(n10)(n12),hyperedge border,draw=brown!70,inner sep=8pt] (G6border) {};
    \end{scope}
    
    \begin{scope}[on background layer]
        \node[fit=(n1)(n7)(n8)(n9),hyperedge,fill=violet!30,inner sep=10pt] (G7region) {};
        \node[fit=(n1)(n7)(n8)(n9),hyperedge border,draw=violet!70,dashed,dash pattern=on 8pt off 4pt,inner sep=10pt] {};
        
        \node[fit=(n4)(n10)(n11),hyperedge,fill=teal!30,inner sep=10pt] (G8region) {};
        \node[fit=(n4)(n10)(n11),hyperedge border,draw=teal!70,dashed,dash pattern=on 8pt off 4pt,inner sep=10pt] {};
    \end{scope}
    
    \node at (G4region.north west) [anchor=north west, font=\small, color=cyan!70, inner sep=4pt] {$G_4$};
    
    \node at (G5region.north west) [anchor=north west, font=\small, color=blue!70, inner sep=4pt] {$G_5$};
    
    \node at (G6region.north east) [anchor=north east, font=\small, color=brown!70, inner sep=4pt] {$G_6$};
    
    \node at (G7region.south west) [anchor=south west, font=\small, color=violet!70, inner sep=4pt] {$G_7$};
    
    \node at (G8region.north east) [anchor=north east, font=\small, color=teal!70, inner sep=4pt] {$G_8$};

    \coordinate (G6anchor) at (G6border.west);
    
    \coordinate (G4anchor) at (G4border.east);

    \draw[comm, ultra thick, red!70] (n9) to[out=20,in=160] (n10);
    \draw[comm, ultra thick, red!70] (n9) to[out=90,in=200] (n12);

\end{tikzpicture}
    \caption{Hypergraph representation of partial separability with $n=13$ nodes and $m=8$ regularizer components. Nodes represent coordinate blocks $x_i$. Generalized edges show regularizer dependencies (orange for pairwise, shaded regions for 3-way and 4-way) and their regularization functions, e.g., $G_6(x)=G_6(x_9,x_{10},x_{12})$. In our randomized local coordination protocol, each node independently samples one regularizer uniformly from $\{1,\ldots,m\}$. For example, node 9 participates in two regularizers: with probability 1/8 it coordinates with nodes $\{10,12\}$, and with probability 1/8 it coordinates with nodes $\{1,7,8\}$. This protocol reduces expected communication to $\sum_{j=1}^m (a_j^2 - a_j)/m$ messages per iteration, where $a_j = |S_j|$ is the size of component $j$.}%
    \label{fig:partial_separability_hypergraph}
\end{figure}
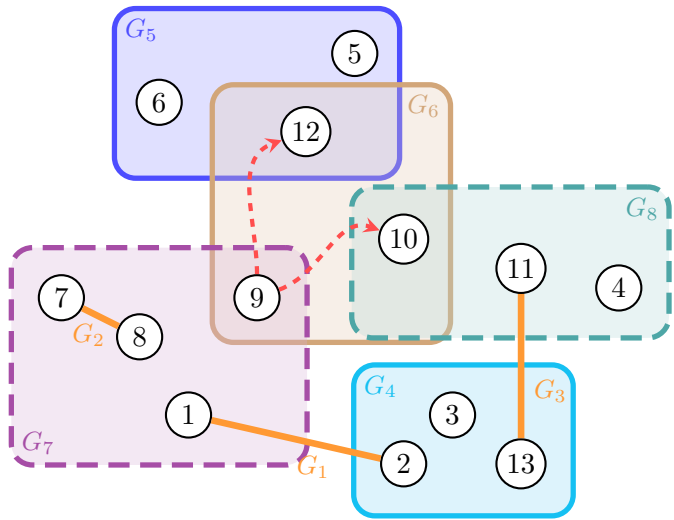

\begin{definition}[Partial Separability]\label{def:partial_separability}
A function $G: \R^{nd} \to \R$ admitting the decomposition $G(x) = \sum_{j=1}^m G_j(x)$ is \emph{partially separable} when each summand $G_j$ exhibits dependence on only a subset of the coordinate blocks of $x$. More precisely, there exists a collection of support sets $\{S_j \subseteq \{1, \ldots, n\} \mid j = 1, \ldots, m\}$ such that $G_j(x) = G_j(x_{S_j})$ for all $j$, where $x_{S_j} = (x_i \mid i \in S_j)$ denotes the restriction of $x$ to the blocks indexed by $S_j$.
\end{definition}

\subsection{Decentralized Multi-Task Learning}\label{subsec: graph_guided}

Our framework applies directly to \emph{decentralized multi-task learning}, where different variable blocks in $x = (x_1, \ldots, x_n)$ each correspond to a distinct learning task. These tasks are distributed across nodes in a network. The objective is to collaboratively learn the solution of~\cref{eq:general_problem} without centralizing data, and thus preserve privacy and computational efficiency~\citep{LZZHZL2017, VWKKVR2020}.

The regularizer $G$ encodes task relationships that improve generalization performance when individual tasks have limited data.
The most prevalent formulations are \emph{graph-guided regularizers} that encode inter-task relationships through network topology~\citep{KX2009}. Given a communication graph $\mathcal{G} = (\mathcal{V}, \mathcal{E})$ with vertices $\mathcal{V} \coloneqq \{1, \ldots, n\}$ representing nodes and edges $\mathcal{E}$ capturing dependencies, the regularizer takes the form
\begin{equation}
\label{eq:graph_guided_regularizer}
G(x) = \textstyle \sum_{(i,j) \in \mathcal{E}} g_{ij}(x_i, x_j).
\end{equation}
The structured optimization problem~\cref{eq:general_problem} with such regularizers reduces to
\begin{equation} 
\label{eq:multi_task_with_graph}
\min_{x_1, \ldots, x_n \in \mathbb{R}^d} \textstyle \ \sum_{i \in \mathcal{V}} f_i(x_i) + \sum_{(i,j) \in \mathcal{E}} g_{ij}(x_i, x_j),
\end{equation}
which makes explicit how the graph structure $\mathcal{G}$ determines both communication topology and regularization dependencies.

For graph-guided regularizers such as the hospital network, each edge $e = (i,j) \in \mathcal{E}$ defines a regularizer $G_e(x) = g_{ij}(x_i, x_j)$ with support set $S_e = \{i, j\}$. Each regularizer couples precisely two hospitals, regardless of the total network size $n$.

\subsection{Coordination Bottleneck}

Consider the application of standard proximal subgradient methods to solve \cref{eq:general_problem} \citep{rockafellar_MonotoneOperatorsProximal_1976a, PB2014, combettes2005signal}. At each iteration $t$, these methods update $x^{(t+1)}$ by alternating between a gradient step with a positive step size $\alpha^{(t)}$,
\begin{subequations}
\begin{equation}
\label{eq:PGD_subgradient_step}
z_i^{(t)} = x_i^{(t)} - \alpha^{(t)} g_i^{(t)}, \quad g_i^{(t)} \in \partial f_i(x_i^{(t)}),
\end{equation}
and a proximal step with positive step sizes $\beta^{(t)}$:
\begin{equation}
\label{eq:PGD_proximal_step}
x^{(t+1)} = \prox_{\beta^{(t)} G}(z^{(t)}) \coloneqq \argmin_{x} \Set{ {\textstyle\sum_{j=1}^m G_j(x)} + \frac1{2\beta^{(t)}} \norm{z^{(t)} - x}^2}.
\end{equation}
\end{subequations}
Block separability of $F$ implies that each node may operate independently and in parallel on local data.
 However, the subsequent proximal step requires global coordination.

Here lies the fundamental obstacle: \emph{the proximal operator of a sum is not the sum of proximal operators}. While each individual collaboration relationship modeled by $G_j$ may admit a simple proximal operator, evaluating the proximal map for the sum $G=\sum_{j=1}^m G_j$ triggers all nodes in $\bigcup_{j=1}^m S_j$ to coordinate simultaneously. When data transmission costs outweigh local computation, such coordination becomes prohibitively expensive~\citep{CBDELPZ2023}.

\subsection{Randomized Local Coordination}

To break the global coupling imposed by the sum structure, we introduce a simple randomized protocol. At each iteration, every node $i$ independently samples a regularizer $G_j$ uniformly at random from $\{1, \ldots, m\}$, that is, $j \sim \text{Unif}(\{1, \ldots, m\})$. Node $i$ communicates with the other nodes in $S_j$ and participates in evaluating the proximal map of $G_j$ only if it particpates in that collaboration, that is, if $i \in S_j$. From another prespective, node $i$ communicates with probability $d_i/m$, where $d_i$ counts components involving that node. Otherwise, node $i$ proceeds without any communication.

Consider the hospital network in \cref{fig:partial_separability_hypergraph}. If Hospital 9 samples $j_9 = 6$, it evaluates only $G_6$ and coordinates exclusively with Hospitals 10 and 12. If Hospital 4 samples a collaboration that does not involve it ($4 \notin S_{j_4}$), it proceeds without any communication. This randomized approach shifts the communication pattern from global synchronization to localized, selective interactions.

Each node coordinates with $\bar{a}-1$ partners when its sampled regularizer is active, where $\bar{a} = \mathbb{E}|S_j|$ denotes the average collaboration size. For graph-guided regularizers (cf.\@ \Cref{subsec: graph_guided}), this yields precisely~2 expected messages per iteration. \Cref{tab:communication_comparison} quantifies the improvement: our method requires $O(1)$ communications while maintaining centralized convergence rates, whereas existing methods either scale with network parameters or sacrifice convergence guarantees.

\subsection{Contributions and Roadmap}

This randomized local coordination strategy leads us to develop \BlockProx, a decentralized optimization method that exploits partial separability to achieve communication efficiency for problem~\cref{eq:general_problem}. Our contributions:

\begin{description}[leftmargin=0cm, style=unboxed, labelindent=0em, align=left, font=\normalfont\itshape]

\item[Decentralized updates] (\Cref{sec:blockprox_method}). The \BlockProx\ algorithm operates without global synchronization, permitting fully decentralized optimization.

\item[Graph-Guided specialization] (\Cref{sec:blockprox_specialization}). \RandomEdge\ specializes \BlockProx\ to graph-guided regularizers defined in \cref{eq:graph_guided_regularizer}, achieving 2 expected messages per iteration.

\item[Communication efficiency] (\Cref{sec:communication_analysis}). Expected total communication equals $\sum_{j=1}^m (a_j^2 - a_j)/m$ messages per iteration, depending only on local coupling sizes $a_j = |S_j|$. Each node $i$ communicates with probability $d_i/m$, where $d_i$ counts components involving that node (\cref{thm:communication}).

\item[Iteration complexity] (\Cref{sec:convergence_theory}). Under bounded subgradients: $\tilde{\BigO}(\varepsilon^{-2})$ iterations are needed to obtain an \( \epsilon \)-solution (\cref{thm:main_theorem_1}); the same iteration complexity also holds for smooth convex \( F \) (\cref{thm:L_smooth_convergence}). For smooth, strongly convex $F$: $\BigO(\varepsilon^{-1})$ to an $\varepsilon$-solution (\cref{thm:sublinear-rate-convergence-strong-convexity}), or $\BigO(\log(1/\varepsilon))$ to a neighborhood (\cref{thm:linear-rate-neighborhood}).

\item[Reproducible empirical validation] (\Cref{sec:numerical_experiments}). Experiments on synthetic and real-world datasets confirm theoretical predictions. The data files and scripts used to generate the numerical results can be obtained at the
URL
\begin{center}
  \url{https://github.com/MPF-Optimization-Laboratory/BlockProx}.
\end{center}
\end{description}

\begin{table}[t]
\centering
\begin{threeparttable}
\caption{Per-iteration communication cost and convergence guarantees for decentralized algorithms.}
\begin{tabular}{@{}lccc@{}}
  \toprule
    {Algorithm\tnote{(a)}} & \makecell{Communication \\ per iteration\tnote{(b)}} & \makecell{Topology \\ dependence} & \makecell{Convergence \\ guarantee\tnote{(c)}} \\
  \midrule
  \rowcolor{gray!20} \RandomEdge\ (ours) & 2 (expected) & No & C, SC, CR \\
  ADMM~\citep{HLB2015} & $4m$ & Yes & C, SC, CR  \\
  DGD~\citep{YLY2016} & $2m$ & Yes & C, SC, TD, CR \\
  Scaffnew~\citep{MMSR2022}\tnote{(e)} & $2pm$ (expected) & Yes & C, SC, TD, CR \\
  Walkman~\citep{MYHGSY2020} & 1\tnote{(f)} & No & SCLS, TD, CR \\
  ESDACD~\citep{HBM2019} & 2 & No & SC, TD \\
  \bottomrule
\end{tabular}

\begin{tablenotes}
  \footnotesize
  \item[(a)] We exclude several centralized and federated methods, including incremental proximal methods~\citep{B2011}, proximal average (ProxAvg)~\citep{Y2013}, RandProx~\citep{CR2023}, and SMPM~\citep{CGR2025}, even though they can solve~\cref{eq:general_problem}.
  Incremental proximal methods maintain a global index of the sampled regularizer and update all involved blocks simultaneously, so they are not decentralized. ProxAvg can be adapted to decentralized settings, but no decentralized analysis is available. RandProx and SMPM are federated algorithms for consensus optimization.

  \item[(b)] The communication cost counts the number of \( \R^d \) messages; \( m \) denotes the number of edges in the network.
  ADMM transfers $(x_i, x_j)$ and auxiliary variables for each edge $(i,j)$, yielding $4m$ messages per iteration.
  DGD uses $2m$ messages as each node exchanges its local variable with all neighbors.
  \textsc{Scaffnew} evaluates the proximal operator with probability $p$; when the central server communicates with all nodes, the expected cost is $2pm$.
  Edge-activation methods operate on single edges: Walkman performs a one-way transfer (1 message); ESDACD performs a bidirectional exchange (2 messages).

  \item[(c)] We use the abbreviations C (convex), SC (strongly convex), SCLS (strongly convex least squares), TD (topology-dependent rate), and CR (centralized rate).
  TD indicates that the convergence rate depends on the network topology rather than solely on \( n \) or \( m \); eliminating dependence on \( n \) and \( m \) is generally infeasible.
  For instance, star and path topologies have the same number of edges but yield different rates for DGD, \textsc{Scaffnew}, Walkman, and ESDACD. Our rate depends only on \( m \), so it is identical for these two topologies.

  \item[(d)] DGD, \textsc{Scaffnew}, Walkman, and ESDACD are designed for consensus optimization, which enforces a common solution across nodes. This contrasts with multitask learning, where nodes may have different variables.
  While consensus methods can be applied to multitask problems, this reduction has drawbacks; see~\cref{subsec:related-work} and~\Cref{sec:numerical_experiments}.

  \item[(e)] Convergence guarantees for \textsc{Scaffnew} in convex settings appear in~\citet{GAYCC2023}.

  \item[(f)] When Walkman is applied to multitask learning by reformulating the problem as consensus with a regularizer \( G \), each iteration must transfer a vector of length \( nd \) containing copies of all variables; the per-iteration communication therefore becomes \( n \), not \( 1 \).
\end{tablenotes}
\label{tab:communication_comparison}
\end{threeparttable}
\end{table}

\subsection{Related Work}
\label{subsec:related-work}

Decentralized optimization divides into consensus optimization that seek a common solution across nodes, and multi-task learning for specific regularizing structures. Neither handles general partially separable problems.

The multi-task formulation~\cref{eq:multi_task_with_graph} encompasses \emph{consensus optimization} as a special case. When the objective is to find a common solution $x_0$ across all nodes, the problem reduces to
\begin{equation}
\label{eq:consensus_problem}
\min_{x_0 \in \mathbb{R}^d} \textstyle F(x) \coloneqq \tfrac{1}{n} \sum_{i=1}^n f_i(x_0).
\end{equation}
In principle, one would convert a consensus problem to a regularized problem via exact penalty functions \citep{han_ExactPenaltyFunctions_1979}, but this requires appropriate regularity conditions and a penalty parameter that is difficult to estimate.

The typical and most commonly employed method for solving consensus optimization problem~\eqref{eq:consensus_problem} is \emph{Decentralized gradient descent (DGD)}~\citep{YLY2016}.
Its convergence for nonconvex problems relies on \emph{uniformly bounded data heterogeneity}, requiring $\sum_{i=1}^n \|\nabla f_i(x) - \nabla F(x)\|^2$ to be  uniformly bounded over $x$~\citep{KLBJS2020, ZLLSCC2018}.
Example~3 in \citet{CHWZL2020} shows DGD diverges for any constant step-size when heterogeneity grows unbounded with $\|x\|$.
Real-world applications often violate uniform boundedness, particularly with multi-source data~\citep{LZT2015, TLISBEMLDC2022}.

Several methods relax the bounded heterogeneity assumption: EXTRA~\citep{SLWY2015} modifies the DGD update scheme, while gradient-tracking techniques~\citep{LS2016, AY2024, SSPY2023} maintain local gradient estimates.  Alternatively, primal-dual methods converge without bounded heterogeneity~\citep{SFMTJJ2018, LLZ2020,CL2018,EPBBA2020, EPBIBA2021}. Among these, \citet{SFMTJJ2018} and \citet{LLZ2020} achieve communication efficiency.

Communication efficient strategies include Newton methods with dynamic average consensus~\citep{LZSL2023,LCCC2020}, event-triggered protocols~\citep{CB2021}, gradient compression~\citep{OSDD2021,EPBIBA2021,LLHP2022}, selective communication skipping~\citep{MMSR2022,CMR2025}, probabilistic device selection~\citep{CSPEC2021,GCYR2022}, and gossip-type methods~\citep{BGP2006,HBM2019,MYHGSY2020}.
The gossip-type methods ESDACD~\citep{HBM2019} and Walkman~\citep{MYHGSY2020} achieves topology-independent constant per-iteration communication cost by randomly selecting an edge in each iteration to communication, which is similar to our methods.
However, ESDACD requires the local loss function to be \( L \)-smooth and strongly convex, making it restrictive; in each iteration of Walkman, only one node is updated using a computationally expensive proximal operator, leading to slow update and high computation cost.

Consensus solutions perform poorly on individual nodes~\citep{ZLLSCC2018}. This limitation and demand for \textit{personalized learning}~\citep{TYCY2023} drive decentralized multi-task approaches.
Multi-task algorithms target specific instances of problem~\cref{eq:decomposed_problem}. Network Lasso~\citep{HLB2015} solves graph-guided regularizers $g_{ij}(x_i, x_j) = w_{ij} \|x_i - x_j\|_2$ via ADMM with closed-form updates. \citet{ZLZ2022} generalize this to arbitrary penalties $P(x_i - x_j)$, but their spanning-tree restriction changes the objective.
Other approaches require smoothness of $F$ or $G$~\citep{KSR2017,CB2020,YC2024}, twice differentiability of both~\citep{CL2020}, or stochastic smooth objectives~\citep{SCDB2024}. These methods do not handle arbitrary partially separable regularizers or adapt communication to sparsity patterns.

Consensus optimization embeds within our framework through penalty regularizers, but the reverse is not efficient~\citep{JMS2018}: forcing multi-task problems into consensus form requires each node to maintain copies of all variables, thereby undermining the benefits of partial separability and significantly increasing memory complexity.
Some consensus methods incorporate graph-guided regularizers~\citep{GSYC2024,ZGSC2025}, but these regularizers embed within local loss functions rather than being shared across nodes.

Recently, motivated by ProxSkip~\citep{MMSR2022} and incremental proximal methods~\citep{B2011}, several communication-efficient methods are proposed for consensus optimization problems with a regularizer, such as RandProx~\citep{CR2023}, LoCoDL~\citep{CMR2025}, SMPM~\citep{CGR2025}.
However, these algorithms are all restrictive: they can only apply to federated settings; LoCoDL further requires both the local loss function and the regularizer to be \( L \)-smooth and strongly convex.
These algorithms are based on an idea similar to our randomized local coordination but actually different: they address the computational difficulties and high communication costs encountered in the computation of the proximal operator of the regularizer by using a random approximation, and further improve communication efficiency by selectively skipping proximal steps.
The key difference is that in these methods, all nodes share a common indicator determining if skipping the proximal steps, while in our method, each node has its own random variable to determine if the local randomized proximal step can be skipped.
Moreover, in our method, each node randomly selects a regularizer component and compute its proximal operator, but only update its local part using the proximal operator.

Our approach exploits partial separability directly: nodes coordinate only when regularizer structure demands it. We achieve communication efficiency for arbitrary partially separable problems without smoothness assumptions or topology restrictions, though we do analyze the convergence rate with and without smoothness.

\section{The BlockProx Method}
\label{sec:blockprox_method}

We present the \BlockProx\ method for any partially separable regularizer (\Cref{sec:blockprox_algorithm,sec:communication_mechanism}), then specialize it to graph-guided regularizers where it achieves constant communication cost (\Cref{sec:blockprox_specialization,sec:communication_analysis}). Throughout this paper, we consider undirected edges $(i,j) \in \cE$ without self-loops.

\subsection{Algorithm Design and Independent Sampling}
\label{sec:blockprox_algorithm}

The algorithm exploits partial separability: each node $i$ samples index $j_i^{(t)} \in \{1,\ldots,m\}$ uniformly at random at each iteration $t$ and coordinates only with nodes in support set $S_{j_i^{(t)}}$. For graph-guided regularizers, this reduces to sampling a single edge, though nodes may receive additional communication requests from neighbors.

\BlockProx\ alternates between the gradient step from~\cref{eq:PGD_subgradient_step} and a randomized proximal step that approximates~\cref{eq:PGD_proximal_step}. Each node $i$ independently samples $j_i^{(t)} \in \{1, \ldots, m\}$ uniformly and updates
\begin{equation}
    \label{eq:proximal_step}
    x_i^{(t+1)} = \begin{cases}
        \bigl[\prox_{\beta^{(t)} G_i^{(t)}}(z^{(t)})\bigr]_i & \text{if } i \in S_{j_i^{(t)}}, \\
        z_i^{(t)} & \text{otherwise,}
    \end{cases}
\end{equation}
where $S_j$ denotes the support set of regularizer $G_j$. For notational convenience, we write $G_i^{(t)} \coloneqq G_{j_i^{(t)}}$ for the regularizer sampled by node $i$ at iteration $t$. The support set $S_{j_i^{(t)}}$ determines node $i$'s communication partners: node $i$ participates in proximal computation only when $i \in S_{j_i^{(t)}}$, which requires communication exclusively within that support set. \Cref{alg:BlockProx} details the complete implementation.

\begin{algorithm2e}[t]
\DontPrintSemicolon
\caption{\BlockProx: decentralized proximal gradient with randomized coordination}
\label{alg:BlockProx}
\KwData{stepsizes $\{\alpha^{(t)}, \beta^{(t)}\}_{t\geq 0}$ such that $\beta^{(t)} = m \alpha^{(t)}$,  initial point $\{x_i^{(0)}\}_{i=1, \ldots, n}$, total iterations $T$}
\For{$t=0, 1, \ldots, T-1$}{
    \ForPar{$i=1, 2, \ldots, n$}{
        \tcp{Gradient Step from~\cref{eq:PGD_subgradient_step}}
        compute $g_i^{(t)} \in \partial f_i(x_i^{(t)})$ and set $z_i^{(t)} = x_i^{(t)} - \alpha^{(t)} g_i^{(t)}$
    }
    $z^{(t)} = (z_1^{(t)}, \ldots, z_n^{(t)})$ \tcp*{Each node $i$ stores its own $z_i^{(t)}$ only}
    \ForPar{$i=1, 2, \ldots, n$}{
        \tcp{Proximal Step}
        sample $j_i^{(t)} \sim \Unif(\{1, \ldots, m \})$ \tcp*{Randomized local coordination}
        \eIf{$i \in S_{j_i^{(t)}}$}{
            $x_i^{(t+1)} \coloneqq \bigl[\prox_{\beta^{(t)} G_i^{(t)}}(z^{(t)})\bigr]_i$\tcp*{Node $i$ retrieves $z_k^{(t)}\ \forall k \in S_{j_i^{(t)}}$}}
            {
            $x_i^{(t+1)} = z_i^{(t)}$\tcp*{No communication required}
            }
    }
}
\end{algorithm2e}
    
\subsection{Communication Protocol}
\label{sec:communication_mechanism}

The proximal operator $\prox_{\beta^{(t)} G_i^{(t)}}$ requires only variables $\{z_k^{(t)}\}_{k \in S_{j_i^{(t)}}}$ to compute node $i$'s update. When node $i$ samples regularizer $j$ with $i \in S_j$, it retrieves $z$ values only with nodes in $S_j$. Particularly, node $i$ retrieves $z_k^{(t)}$ for all $k\in S_j$. If $i \notin S_j$, then $x_i^{(t+1)} = z_i^{(t)}$ without communication. Each node stores its local variables $(x_i, z_i)$ and has access to all regularizers $\{G_j \mid i \in S_j\}$.

\section{Specialization: The RandomEdge Algorithm}
\label{sec:blockprox_specialization}

We specialize \BlockProx\ to graph-guided regularizers~\cref{eq:graph_guided_regularizer}, where each edge $(i,j) \in \mathcal{E}$ defines a symmetric regularizer $G_e(x) = g_{ij}(x_i, x_j) = g_{ij}(x_j, x_i)$ with support set $S_e = \{i,j\}$. The resulting method, \RandomEdge, is fully described in \cref{alg:RandomEdge}.

\begin{algorithm2e}[t]
\DontPrintSemicolon
\caption{\RandomEdge: application of \BlockProx\ to graph-guided regularizers}
\label{alg:RandomEdge}
\KwData{stepsizes $\{\alpha^{(t)}, \beta^{(t)}\}_{t \geq 0}$, initial point $\{x_i^{(0)}\}_{i=1}^n$, graph $\cG = (\cV, \cE)$, total iterations $T$}
\For{$t=0, 1, \ldots, T-1$}{
    \ForPar{$i=1, 2, \ldots, n$}{
        \tcp{Subgradient Step}
        compute $g_i^{(t)} \in \partial f_i(x_i^{(t)})$ and set $z_i^{(t)} = x_i^{(t)} - \alpha^{(t)} g_i^{(t)}$\;
    }
    $z^{(t)} = (z_1^{(t)}, \ldots, z_n^{(t)})$\;
    \ForPar{$i=1, 2, \ldots, n$}{
        \tcp{Proximal Step with Randomized Local Coordination}
        $p_i = \deg(i) / |\cE|$\tcp*{Node degree determines coordination probability}
        sample $\theta_i^{(t)} \in \{0, 1\}$ with $\Pr(\theta_i^{(t)} = 1) = p_i$\tcp*{Randomly coordination}
        \eIf{$\theta_i^{(t)} = 1$}{
            \tcp{Coordinate with nodes in selected support set}
            sample $j_i^{(t)} \sim \Unif(\{j \mid i \in S_j\})$\tcp*{Sample edge incident to node $i$}
            retrieve $z_k^{(t)}$ from the neighbor $k \in S_{j_i^{(t)}} \setminus \{i\}$\;
            $w_{S_{j_i^{(t)}}} = \prox_{\beta^{(t)} G_{j_i^{(t)}}}(z_{S_{j_i^{(t)}}}^{(t)})$\tcp*{Evaluate edge regularizer jointly}
            $x_i^{(t+1)} = w_i$\;}
            {$x_i^{(t+1)} = z_i^{(t)}$ \tcp*{Independent update when coordination not needed}}
    }
}
\end{algorithm2e}

The subgradient step matches \BlockProx. In the proximal step, node $i$ coordinates with probability $p_i = \deg(i)/|\mathcal{E}|$. When coordinating, it samples a neighbor uniformly and jointly evaluates the corresponding edge regularizer. 
Note that in the graph setting, the set $\set{j | i \in S_j}$ consists precisely of edges incident to node $i$, making the support set notation a natural generalization of graph neighborhoods (see \cref{fig:node2_coordination}).
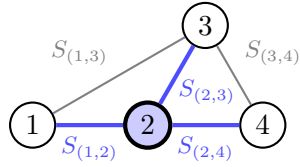
\begin{figure}[t]
\centering
\begin{tikzpicture}[scale=0.85]
    \tikzstyle{vertex}=[circle,draw,thick,minimum size=0.6cm,inner sep=1pt,fill=white]
    \tikzstyle{highlighted}=[circle,draw,ultra thick,minimum size=0.6cm,inner sep=1pt,fill=blue!20]
    \tikzstyle{edge}=[thick,gray]
    \tikzstyle{highlight edge}=[ultra thick,blue!70]
    
    \node[highlighted] (n2) at (0,0) {2};
    \node[vertex] (n1) at (-1.8,0) {1};
    \node[vertex] (n3) at (0.9,1.56) {3};
    \node[vertex] (n4) at (1.8,0) {4};
    
    \draw[highlight edge] (n1) -- (n2) node[midway,below] {\footnotesize $S_{(1,2)}$};
    \draw[highlight edge] (n2) -- (n3) node[pos=0.2,right] {\footnotesize $S_{(2,3)}$};
    \draw[highlight edge] (n2) -- (n4) node[midway,below] {\footnotesize $S_{(2,4)}$};
    \draw[edge] (n1) -- (n3) node[midway,above left] {\footnotesize $S_{(1,3)}$};
    \draw[edge] (n3) -- (n4) node[midway,above right] {\footnotesize $S_{(3,4)}$};
\end{tikzpicture}
\vspace{-0.5em}
\caption{Node 2's incident edges in the graph specialization. Blue edges show $e \in \mathcal{E}$ where $2 \in e$, illustrating that node 2 coordinates with probability $p_2 = \deg(2)/|\mathcal{E}| = 3/5$.}
\label{fig:node2_coordination}
\end{figure}
For instance, consider a 4-node graph with edges $\mathcal{E} = \{(1,2), (1,3), (2,3), (2,4), (3,4)\}$ as shown in \cref{fig:node2_coordination}. 
Node 2 participates in edges $(1,2)$, $(2,3)$, and $(2,4)$, 
yielding $\set{e \in \mathcal{E} | 2 \in e} = \{(1,2), (2,3), (2,4)\}$ and $p_2 = \deg(2)/|\mathcal{E}| = 3/5$.

The following proposition establishes that \RandomEdge\ specializes \BlockProx\ for graph-guided regularizers~\cref{eq:graph_guided_regularizer}, and thus inherits its theoretical guarantees.

\begin{proposition}[\RandomEdge\ and \BlockProx\ Equivalence]
\label{prop:equivalence}
\Cref{alg:BlockProx,alg:RandomEdge} are equivalent when the regularizer $G$ is graph-guided as in~\cref{eq:graph_guided_regularizer}; that is, they generate identical sequences of iterates.
\end{proposition}

\begin{proof}
In the graph-guided setting, each edge $e = \{i,j\} \in \mathcal{E}$ corresponds to a regularizer $G_e(x) = g_{ij}(x_i, x_j)$ with support $S_e = \{i,j\}$. Because the subgradient steps match, we analyze only the proximal step.

For any node $i$ and neighbor $j$, we compute the probability that $i$ updates using edge $\{i,j\}$. For \BlockProx, node $i$ samples one edge uniformly from all $|\mathcal{E}|$ edges. Thus, the probability of selecting edge $\{i,j\}$ is $1/|\mathcal{E}|$. For \RandomEdge, node $i$ coordinates with probability $p_i = \deg(i)/|\mathcal{E}|$ and chooses neighbor $j$ uniformly from its $\deg(i)$ neighbors. Thus, the probability of selecting edge $\{i,j\}$ is
\[
p_i \cdot \frac{1}{\deg(i)} = \frac{\deg(i)}{|\mathcal{E}|} \cdot \frac{1}{\deg(i)} = \frac{1}{|\mathcal{E}|}.
\]
Both algorithms therefore assign probability $1/|\mathcal{E}|$ to each edge activation.
\end{proof}

\section{Communication Analysis}
\label{sec:communication_analysis}

We establish that communication cost depends solely on the regularizer's structure, not on network topology. The \BlockProx\ framework's communication cost depends on the partial separability structure of regularizer $G$ because each component $G_j$ induces communication among nodes in its support set $S_j$. When node $i$ samples component $j$, it coordinates with $|S_j|-1$ other nodes if $i \in S_j$, and with no nodes otherwise. This randomized local coordination mechanism yields the following characterization.

\begin{theorem}[Expected Communication Cost]
\label{thm:communication}
Let $G$ be partially separable as in \cref{def:partial_separability} with support sets $S_j \subseteq \{1, \ldots, n\}$ for $j = 1, \ldots, m$. If $a_j = |S_j|$ for each~$j$, then the expected number of communications across all nodes in each iteration of \BlockProx\ is $$\textstyle\sum_{j=1}^m (a_j^2 - a_j)/m.$$
Moreover, node $i$ communicates with probability $d_i/m$, where $d_i = \sum_{j=1}^m \mathbf{1}\{i \in S_j\}$ counts the number of components that include node $i$.
\end{theorem}

\begin{proof}
Let $\delta_{ij} = \mathbf{1}\{i \in S_j\}$ indicate whether node $i$ participates in component $j$. Then $a_j = |S_j| = \sum_{i=1}^n \delta_{ij}$.

When node $i$ samples component $j$, it communicates with exactly $\delta_{ij}(a_j-1)$ other nodes: if $i \in S_j$, it coordinates with the $a_j-1$ other nodes in $S_j$; otherwise, it performs no communication. Because each component is sampled with probability $1/m$, node $i$'s expected communication count equals
$$
\textstyle(1/m)\sum_{j=1}^m  \delta_{ij} (a_j-1).
$$
Summing over all nodes yields the total expected communications:
\begin{align*}
\sum_{i=1}^n \sum_{j=1}^m \frac{1}{m} \delta_{ij} (a_j-1) 
&= \frac{1}{m} \sum_{j=1}^m (a_j-1) \sum_{i=1}^n \delta_{ij}
= \frac{1}{m} \sum_{j=1}^m (a_j^2 - a_j),
\end{align*}
where the last equality follows from the definition of $a_j$.

To compute the probability that node $i$ communicates in a given iteration, note that node $i$ communicates in a given iteration if and only if it samples a component $j$ such that $i \in S_j$. This occurs when the sampled index $j_i^{(t)}$ satisfies $\delta_{ij_i^{(t)}} = 1$. Since each component is sampled uniformly at random, the probability that node $i$ communicates equals
$$
\Pr[\text{node } i \text{ communicates}] = \sum_{j=1}^m \Pr[j_i^{(t)} = j] \cdot \delta_{ij} = \frac1m\sum_{j=1}^m  \delta_{ij} = \frac{d_i}{m},
$$
where $d_i = \sum_{j=1}^m \delta_{ij}$ counts the number of components that include node $i$.
\end{proof}

The formula $a_j^2 - a_j$ counts the pairwise communications required when $a_j$ nodes coordinate. When component $G_j$ is activated (sampled by at least one node), all $a_j$ nodes in support set $S_j$ must coordinate to compute the proximal map of $G_j$. This coordination forms a complete communication graph among these nodes, requiring every node to communicate with every other node in $S_j$.

\subsection{Graph-Guided Regularizers}
\label{sec:graph-guided-communication}

For graph-guided regularizers, each edge $e = \{i,j\} \in \mathcal{E}$ corresponds to a regularizer $G_e$ with support $S_e = \{i,j\}$, giving $a_j = |S_j| = 2$ for all $j$. Consider \RandomEdge\ in this setting: each node independently samples one edge from the graph, and communication occurs only when two neighboring nodes sample their shared edge.

Substituting $a_j = 2$ into the formula from \Cref{thm:communication} yields
$$
\textstyle\sum_{j=1}^m (a_j^2 - a_j)/m = \sum_{j=1}^m (4 - 2)/m = \sum_{j=1}^m 2/m = 2.
$$
Thus \RandomEdge\ requires exactly 2 expected communications per iteration across all nodes, regardless of the graph's structure or size. To understand the per-node behavior, let $p_i = \deg(i)/m$ denote the probability that node $i$ requires coordination, which occurs when it samples any edge incident to itself. The expected number of communications for node $i$ per iteration equals $p_i$, and summing over all nodes yields 
$$
\tsum_{i=1}^n p_i = \tsum_{i=1}^n \deg(i)/m = 2m/m = 2,
$$
where we use the handshaking lemma: $\sum_{i=1}^n \deg(i) = 2m$ for undirected graphs. We record this result in the following corollary.

\begin{corollary}[Graph-Guided Communication]
\label{cor:randedge_communication}
For graph-guided regularizers where each component $G_j$ couples exactly two nodes (i.e., $a_j = |S_j| = 2$ for all $j$), the \RandomEdge\ algorithm achieves exactly 2 expected communications per iteration across all nodes. Moreover, each node communicates with probability $\deg(i)/m$, where $\deg(i)$ is the degree of node $i$ in the graph.
\end{corollary}

\begin{remark}[Per-Node Communication Variability]
Individual nodes experience variable communication loads: node $i$ may communicate $k$ times if it samples an edge and $k-1$ neighbors sample edges involving $i$. High-degree nodes have proportionally higher expected communication.
\end{remark}

\section{Convergence Theory} \label{sec:convergence_theory}

We now establish convergence guarantees for \BlockProx\ that match centralized stochastic gradient methods despite the algorithm's decentralized nature and minimal communication overhead.

\subsection{Assumptions and Preliminaries}
\label{sec:assumptions}

We analyze the convergence of \BlockProx\ under three settings. For convenience, we restate the \BlockProx\ iterative process. Beginning with $x^{(0)}$ and $t=0$, the algorithm generates iterates via
\begin{subequations} \label{eq:blockprox-iterates}
\begin{align}
z^{(t)} &= x^{(t)} - \alpha^{(t)} g^{(t)} \in \R^{nd}, \quad g^{(t)} = (g_1^{(t)}, \ldots, g_n^{(t)}), \quad g_i^{(t)} \in \partial f_i(x_i^{(t)}) 
\label{eq:subgradient_step_2} \\
x_i^{(t+1)} &= \bigl[ \prox_{\beta^{(t)} G_{j_i^{(t)}}}(z^{(t)}) \bigr]_i \in \R^{d}, \quad j_i^{(t)} \sim \Unif(\{1, \ldots, m\})
\label{eq:proximal_step_local}
\end{align}
\end{subequations}
for all $i \in \{1, \ldots, n\}$, where the update implicitly sets $x_i^{(t+1)} = z_i^{(t)}$ when $i \notin S_{j_i^{(t)}}$, as established in~\cref{eq:proximal_step}.
\begin{assumption}
\label{assumption:main}
The following conditions hold:
\begin{enumerate}[label=(\alph*),leftmargin=1em, style=unboxed, labelindent=0em, align=left]
    \item\label{item:assum-solution}
    The solution set $\mathcal{X}^*$ to \eqref{eq:general_problem} is nonempty.
    
    \item\label{item:assum-reg} The following holds: $\norm{s} \le c$ for every $s \in \partial G_j(x)$, $x \in \R^{nd}$, and $j \in \{1, \ldots, m\}$.
    \item\label{item:assum-loss} Exactly one of the following holds for the loss function $F$:
    \begin{enumerate}[label=(\roman*),leftmargin=1em, style=unboxed, labelindent=1em, align=left]
        \item\label{item:bded-subgrad-F} $F$ has bounded subgradients: $\norm{s} \le c$ for every $s \in \partial F(x)$ and $x \in \R^{nd}$;
        \item\label{item:smooth-F} $F$ is $L$-smooth (differentiable with $L$-Lipschitz gradient);
        \item\label{item:strong-convex-F} $F$ is $L$-smooth and $\mu$-strongly convex.
    \end{enumerate}
\end{enumerate}
\end{assumption}

Observe that the uniform boundedness condition in \cref{assumption:main}\ref{item:assum-reg} is equivalent to requiring each $G_j$ to be $c$-Lipschitz continuous \citep[Theorem 24.7]{R70}.
\Cref{assumption:main}\ref{item:assum-reg} and \cref{assumption:main}\ref{item:assum-loss}\ref{item:bded-subgrad-F} together lead to the bounded subgradient assumption, standard in centralized and decentralized optimization literature; see, for example,~\citet{B2011} for centralized settings, and~\citet{nedic_distributed_2009,SNV2010,JXM2014,YLZZ2021,YHHX2021,YPS2022,LLH2024} for decentralized settings.
It is also common to assume strong convexity and \( L \)-smoothness of the objective function in decentralized optimization~\citep{SLWY2015,YLY2016,SFMTJJ2018,LCCC2020,LLZ2020,MMSR2022,AY2024}.

\subsection{Core Technical Lemmas}
\label{sec:technical_lemmas}

This section establishes cross-cutting mathematical machinery used throughout our convergence analysis. The central challenge is understanding how the expected behavior of our randomized regularizer selection (local coordination) relates to the deterministic proximal-gradient method.

Let $\cF_k \coloneqq \sigma(\{j_i^{(t)} \mid 1 \le i \le n \text{~and~} 0 \le t < k\})$ denote the $\sigma$-algebra generated by all random selections up to iteration $k$, and use the shorthand $\E^k\coloneqq\E[~\cdot\mid\cF_k]$ for the conditional expectation given $\cF_k$.

To support our convergence analysis, we recall a basic \emph{nonexpansiveness-type} property of proximal operators for convex functions.

\begin{lemma}[Nonexpansive Property of Proximal Operator]
\label{lem:proximal}
Let $G_j: \R^{nd} \to \R$ be convex. If $\hat{x} = \prox_{\beta G_j} (z)$ with $\beta > 0$ and $z \in \R^{nd}$, then for any $y \in \R^{nd}$,
$$
\norm{\hat{x} - y}^2 \le \norm{z - y}^2 - 2 \beta \left( G_j(\hat{x}) - G_j(y) \right).
$$
\end{lemma}
\begin{proof}
By \citet[Theorem 6.39(i),(iii)]{beck_FirstOrderMethodsOptimization_2017}, we have
\begin{align*}
    2\langle z-\hat x, y-\hat x \rangle \leq 2\beta(G_j(y) - G_j(\hat x)).
\end{align*}
Using $2\langle a, b \rangle = \| a\|^2 + \| b\|^2 - \| a-b\|^2$ with $a=z-\hat x$, $b=y-\hat x$, then removing $\| z-\hat x\|^2$ from the left-hand side, and rearranging gives the result.
\end{proof}
\subsubsection{Expected Error Decomposition}

The following proposition shows that \BlockProx's random regularizer selection yields an expected squared error equal to the average over all possible selections. This result allows us to reduce the analysis of the new method to that of the incremental gradient methods \citep{B2011}.

\begin{proposition}[Expected Error Decomposition]
\label{lem:decompose_expected_error}
For sequences $\{x^{(k)}, z^{(k)}\}$ generated by \BlockProx\ via~\cref{eq:blockprox-iterates}, we have, for any $k \ge 0$ and any $y \in \R^{nd}$,
\[
\E^k \norm{x^{(k+1)} - y}^2 = \frac1m \tsum_{j=1}^m \norm{\prox_{\beta^{(k)} G_{j}} (z^{(k)}) - y}^2.
\]
\end{proposition}

\begin{proof}
We use the independence of node sampling to deduce that the expected squared distance decomposes into an average over all regularizer selections. Let $y \coloneqq (y_1, y_2, \ldots, y_n)$, where each $y_i \in \R^d$ for $i = 1, 2, \ldots, n$. Because the squared norm is block separable,
\begin{align*}
\E^k \norm{x^{(k+1)} - y}^2 &= \E^k \tsum_{i=1}^n \norm{x^{(k+1)}_i - y_i }^2  \\
&= \tsum_{i=1}^n \E_{j_1^{(k)}, \ldots, j_n^{(k)}} \norm{x^{(k+1)}_i - y_i}^2  \\
&= \tsum_{i=1}^n \E_{j_i^{(k)}} \norm{\bigl[ \prox_{\beta^{(k)} G_i^{(k)}} (z^{(k)}) \bigr]_i - y_i}^2 ,
\end{align*}
where the last equality follows from the independence of each $j_i^{(k)}$ and the fact that $x_i^{(k+1)}$ depends only on $j_i^{(k)}$.
Since $j_i^{(k)}$ is uniform over $\{1, 2, \ldots, m\}$,
$$
\E_{j_i^{(k)}} \norm{\bigl[ \prox_{\beta^{(k)} G_i^{(k)}} (z^{(k)}) \bigr]_i - y_i}^2 
=
\frac1m \tsum_{j=1}^m \norm{\bigl[ \prox_{\beta^{(k)} G_{j}} (z^{(k)}) \bigr]_i - y_i}^2.
$$
Combining these equalities yields
\begin{align*}
\E^k \norm{x^{(k+1)} - y}^2 &= \tsum_{i=1}^n \frac1m \tsum_{j=1}^m \norm{\bigl[ \prox_{\beta^{(k)} G_{j}} (z^{(k)}) \bigr]_i - y_i}^2 \\
&=\frac1m \tsum_{j=1}^m \tsum_{i=1}^n \norm{\bigl[ \prox_{\beta^{(k)} G_{j}} (z^{(k)}) \bigr]_i - y_i}^2 \\
&= \frac1m \tsum_{j=1}^m \norm{\prox_{\beta^{(k)} G_{j}} (z^{(k)}) - y}^2,
\end{align*}
where the final step uses block separability over the index $i$ of the squared norm.
\end{proof}

\subsubsection{Analysis for Proximal Steps}

Define the following vectors related to the proximal step:
\begin{subequations} \label{eq:x-g-hat-definitions}
\begin{align}
   \label{eq:def_hat_x}
   \hat{x}^{(k, j)} &\coloneqq \prox_{\beta^{(k)} G_j} (z^{(k)}),
\\ \label{eq:def_hat_g}
   \hat{g}^{(k,j)} &\coloneqq \frac{1}{\beta^{(k)}} (z^{(k)} - \hat{x}^{(k, j)}) \in \partial G_j(\hat{x}^{(k, j)}).
\end{align}
\end{subequations}
The quantity $\hat{x}^{(k, j)}$ is the update when all nodes select regularizer $j$. Substituting this notation into \cref{lem:decompose_expected_error} yields the working form
\[
\E^k \norm{x^{(k+1)} - y}^2  = \frac1m \tsum_{j=1}^m \norm{\hat{x}^{(k, j)} - y}^2.
\]
Moreover, \cref{assumption:main}\ref{item:assum-reg} gives the bound $\norm{\hat{g}^{(k, j)}} \le c$.

The following lemma states \BlockProx\ achieves expected descent in the proximal step despite using only partial regularizer information at each iteration.
\begin{lemma}[Expected Descent in Proximal Step]
\label{lem:proximal_step}
Under \cref{assumption:main}, the \BlockProx\ algorithm generates sequences $\{x^{(k)}, z^{(k)}\}$ that satisfies, for any $k \ge 0$ and any $y \in \R^{nd}$,
\begin{align*}
\E^k \norm{x^{(k+1)} - y}^2 &\le \norm{z^{(k)} - y}^2 - 2 \alpha^{(k)} \big[ G(x^{(k)}) - G(y) \big] \\
&\qquad + 2 mc (\alpha^{(k)})^2 \norm{g^{(k)}} + 2 m^2 c^2 (\alpha^{(k)})^2.
\end{align*}
\end{lemma}

\begin{proof}
Use the definitions of $\hat{x}^{(k, j)}$ and $\hat{g}^{(k, j)}$ in~\cref{eq:x-g-hat-definitions}.  
We deduce
\begin{align*}
\norm{\hat{x}^{(k, j)} - y}^2 &\overset{(i)}{\le} \norm{z^{(k)} - y}^2 - 2 \beta^{(k)} \big[ G_j(\hat{x}^{(k, j)}) - G_j(y) \big] \\
&\overset{(ii)}{\le} \norm{z^{(k)} - y}^2 - 2 \beta^{(k)} \big[ G_j(x^{(k)}) - G_j(y) \big] \\
&\qquad + 2 \beta^{(k)} c \norm{x^{(k)} - \hat{x}^{(k, j)}},
\end{align*}
where (i) follows from \cref{lem:proximal}, (ii) uses the $c$-Lipschitz continuity of $G_j$.
By triangle inequality and the definitions from~\cref{eq:x-g-hat-definitions}, we have
\begin{align*}
    \norm{x^{(k)} - \hat{x}^{(k, j)}} &\leq \norm{z^{(k)} - x^{(k)}} + \norm{\hat{x}^{(k, j)} - z^{(k)}} \\
    &\leq \alpha^{(k)} \norm{g^{(k)}} + \beta^{(k)} \norm{\hat{g}^{(k, j)}}.
\end{align*}
Then, using $\beta^{(k)} = m \alpha^{(k)}$ and $\norm{\hat{g}^{(k, j)}} \le c$ from \cref{assumption:main}\ref{item:assum-reg},
\begin{align*}
\norm{\hat{x}^{(k, j)} - y}^2 &\le \norm{z^{(k)} - y}^2 - 2 m \alpha^{(k)} \big[ G_j(x^{(k)}) - G_j(y) \big] \\
&\qquad + 2 m (\alpha^{(k)})^2 \norm{g^{(k)}} + 2 m^2 c^2 (\alpha^{(k)})^2.
\end{align*}
Plugging in the result of \cref{lem:decompose_expected_error} on the left-hand side concludes the proof.
\end{proof}

\subsubsection{Last Iterate Bound}

Here we give a bound that relates the final iterate to the average performance through a correction term with three key behaviors.
When the sequence $\theta_h$ has low variance (i.e., $\theta_h \approx \theta_{k-t}$), the correction term vanishes, yielding the standard averaging bound.
The weights $\omega_t = (t^2+t)^{-1}$ form a telescoping series summing to less than~1, preventing the correction term from dominating.
This lemma transforms per-iteration descent into explicit convergence rates, with the correction term vanishing at the optimal rate through our choice of step size $\alpha_k$. We give in \Cref{app:proof_last_iteration} a proof adapted to our context.

\begin{lemma}[Last Iteration Bound~\citep{Orabona2020}]
\label{lem:last_iteration}
For any sequence of positive scalar pairs $\set{\alpha_k, \theta_k}_{k=1,2,\ldots}$, where $\alpha_k$ in particular is non-increasing,
\[
\alpha_k \theta_k \le \frac{1}{k+1} \sum_{t=0}^k \alpha_t\theta_t + \sum_{t=1}^{k} \omega_t \sum_{h=k+1-t}^{k} \alpha_h \left( \theta_h - \theta_{k-t} \right).
\]
\end{lemma}

\subsection{Convergence under Bounded Subgradients}
\label{sec:bounded_subgradients}

Under bounded subgradients, \BlockProx\ achieves last-iterate convergence with the best-known $\tilde{\BigO}(1/\sqrt{k})$ rate for general convex objectives~\citep{SZ2013}.

The next lemma states that when both $F$ and $G$ have bounded subgradients, \BlockProx\ achieves expected descent despite using only partial regularizer information at each iteration.

\begin{lemma}[Expected Descent under Bounded Subgradients]
\label{lem:main_inequality_bounded_subg}
Let \cref{assumption:main} hold with Case \ref{item:assum-loss}\ref{item:bded-subgrad-F} (bounded subgradients of $F$). Then \BlockProx\ generates a sequence $\{x^{(k)}\}$ that satisfies, for any $k$ and any $y \in \R^{nd}$,
\[
\E^k \norm{x^{(k+1)} - y}^2 \le \norm{x^{(k)} - y}^2 - 2 \alpha^{(k)} [ H(x^{(k)}) - H(y) ] + (2m^2 + 2m + 1) (\alpha^{(k)} c)^2.
\]
\end{lemma}

\begin{proof}
\cref{lem:proximal_step} and \cref{assumption:main}\ref{item:assum-loss}\ref{item:bded-subgrad-F} give
\[
\E^k \norm{x^{(k+1)} - y}^2 \le \norm{z^{(k)} - y}^2 - 2 \alpha^{(k)} [ G(x^{(k)}) - G(y) ] + (2m^2 + 2m) (\alpha^{(k)} c)^2.
\]
Apply the subgradient descent property (see, e.g., \citep[Lemma 8.11]{beck_FirstOrderMethodsOptimization_2017}) for $z^{(k)} = x^{(k)} - \alpha^{(k)} g^{(k)}$ to deduce
\[
\norm{z^{(k)} - y}^2 \le \norm{x^{(k)} - y}^2 - 2 \alpha^{(k)} \big[ F(x^{(k)}) - F(y) \big] + (\alpha^{(k)} c)^2.
\]
Substituting this bound into the first inequality yields
\begin{align*}
\E^k \norm{x^{(k+1)} - y}^2 
&\le \norm{x^{(k)} - y}^2 - 2 \alpha^{(k)} [ H(x^{(k)}) - H(y) ] \\
&\qquad+ (2m^2 + 2m + 1) (\alpha^{(k)} c)^2.
\end{align*}
The proof is completed.
\end{proof}

We now prove that \BlockProx\ converges at a rate that matches, up to a logarithmic term, the fundamental limit for stochastic gradient methods on general convex functions \citep{nemirovski1983problem}.

\begin{theorem}[Convergence under Bounded Subgradients]
\label{thm:main_theorem_1}
Let \cref{assumption:main} hold with Case~\ref{item:assum-loss}\ref{item:bded-subgrad-F} (bounded subgradients of $F$). Let $\alpha^{(k)} = \alpha/\sqrt{k+1}$ for some positive constant $\alpha$.
Then \BlockProx\ generates a sequence $\{x^{(k)}\}$ that satisfies, for any $k$,
\begin{equation}
\label{eq:main_theorem_1}
\E \bigl[ H(x^{(k)}) - H^* \bigr] \le \frac{1}{2\sqrt{k+1}} \left( \frac{1}{\alpha} \dist(x^{(0)}, \mathcal{X}^*)^2 + \BigO(\alpha c^2 m^2 \log k)  \right),
\end{equation}
where $\dist(x^{(0)}, \mathcal{X}^*) = \min_{x^* \in \mathcal{X}^*} \Norm{x^{(0)} - x^*}$, and the term $\BigO(\alpha c^2 m^2 \log k)$ can be made explicit as $3(2m^2 + m + 1) \alpha c^2 (1 + 1 \log(k+1))$.
\end{theorem}

\begin{proof}
It is sufficient to prove, for any $x^* \in \mathcal{X}^*$,
\begin{equation}
\label{eq:main_theorem_1_eqiv}
\E \bigl[ H(x^{(k)}) - H^* \bigr] \le \frac{1}{2\sqrt{k+1}} \left( \frac{1}{\alpha} \Norm{x^{(0)} - x^*}^2 + \BigO(\alpha c^2 m^2 \log k) \right).
\end{equation}

\emph{Step 1: Apply the Last Iteration Bound.}
Apply \cref{lem:last_iteration} with the identifications $\alpha_k\equiv\alpha^{(k)}$ and $\theta_k \equiv \E\bigl[H(x^{(k)}) - H^*\bigr]$. Because $\alpha^{(k)}$ is non-increasing,
\begin{align}
\alpha^{(k)} \E \bigl[ H(x^{(k)}) - H^* \bigr]
&\le \frac{1}{k+1} \sum_{t=0}^k \alpha^{(t)} \E \bigl[ H(x^{(t)}) - H^* \bigr] \nonumber \\
&\qquad + \sum_{t=1}^{k} \frac{1}{t^2+t} \sum_{h=k+1-t}^{k} \alpha^{(h)} \E \bigl[ H(x^{(h)}) - H(x^{(k-t)}) \bigr]. \label{aligned_15}
\end{align}
Use \cref{lem:main_inequality_bounded_subg} and apply the law of iterated expectations to deduce that for any $h$ and $y \in \R^{nd}$,
\begin{equation}
\label{aligned_16}
2 \alpha^{(h)} \E \bigl[ H(x^{(h)}) - H(y) \bigr] \le \E \Norm{x^{(h)} - y}^2 - \E \Norm{x^{(h+1)} - y}^2 + \big(\alpha^{(h)}\big)^2 C_1,
\end{equation}
where $C_1 = (2m^2 + m + 1)c^2$.

\emph{Step 2: Telescope the objective gap.}
For any $x^* \in \mathcal{X}^*$, setting $y = x^{*}$ in~\cref{aligned_16} and summing over $h = 0, 1, \ldots, k$ yields
\begin{align}
\sum_{h=0}^k \alpha^{(h)} \E \bigl[ H(x^{(h)}) - H^* \bigr] &\le \half \Norm{x^{(0)} - x^*}^2  + \half \sum_{h=0}^k \big(\alpha^{(h)}\big)^2 C_1.
\label{aligned_17}
\end{align}
To bound the sum of squared stepsizes, we use the standard integral approximation of the harmonic series:
$$
\sum_{h=0}^k (\alpha^{(h)})^2 = \alpha^2 \sum_{h=1}^{k+1} \frac{1}{h} \le  \alpha^2 \left( 1 + \log(k+1) \right), 
$$
and by~\cref{aligned_17},
\begin{equation}
\label{aligned_18}
\tfrac{1}{k+1} \sum_{t=0}^k \alpha^{(t)} \E \bigl[ H(x^{(t)}) - H^* \bigr] \le \tfrac{1}{2(k+1)} \left( \Norm{x^{(0)} - x^*}^2 + \big( 1 + \log(k+1) \big) \alpha^2 C_1 \right).
\end{equation}
The bound combines initial error with accumulated variance from random sampling.

\emph{Step 3: Bound accumulated variance.}
Setting $y = x^{(k-t)}$ in~\cref{aligned_16} and summing over $h = k-t, k-t+1, \ldots, k$,
$$
\begin{aligned}
  \sum_{\mathclap{h=k-t}}^{k} &\alpha^{(h)} \E \bigl[ H(x^{(h)}) - H(x^{(k-t)}) \bigr] \\
  \le & - \half \E \Norm{x^{(k+1)} - x^{(k-t)}}^2 + \half \sum_{\mathclap{h=k-t}}^{k} \big(\alpha^{(h)}\big)^2 C_1.
\end{aligned}
$$
Since $H(x^{(h)}) - H(x^{(k-t)}) = 0$ for $h = k-t$, the above inequality simplifies to
\begin{align}
\sum_{\mathclap{h=k+1-t}}^{k} \alpha^{(h)} \E \left[ H(x^{(h)}) - H(x^{(k-t)}) \right] \le \half \sum_{\mathclap{h=k-t}}^{k} \big(\alpha^{(h)}\big)^2 C_1. \label{aligned_19}
\end{align}
We bound the sum of squared stepsizes using the logarithmic inequality $\log(1+\xi) \le \xi$ for $\xi\ge0$, similar to \cite{Orabona2020}:
\begin{align*}
\sum_{h=k-t}^{k} (\alpha^{(h)})^2 &= \alpha^2 \sum_{h=k+1-t}^{k+1} \frac{1}{h} \le \alpha^2 \left( \frac{1}{k+1-t} + \int_{k+1-t}^{k+1} \frac{1}{\xi} d \xi \right) \\
&= \alpha^2 \left( \frac{1}{k+1-t} + \log \left( 1 + \frac{t}{k+1-t} \right) \right) \\
&\le \alpha^2 \left( \frac{1}{k+1-t} + \frac{t}{k+1-t} \right) = \alpha^2 \left( \frac{t+1}{k+1-t} \right).
\end{align*}
Similarly, for the double sum, similar to \cite{Orabona2020} we have:
\begin{align*}
\sum_{t=1}^{k} \frac{1}{t(k+1-t)} &= \sum_{t=1}^{k} \frac{1}{k+1} \left( \frac{1}{t} + \frac{1}{k+1-t} \right) = \frac{1}{k+1} \left(\sum_{t=1}^{k} \frac{1}{t} + \sum_{t=1}^k \frac{1}{k+1-t}\right) \\
&= \frac{1}{k+1} \left(\sum_{t=1}^{k} \frac{1}{t} + \sum_{t=1}^{k} \frac{1}{t}\right) = \frac{2}{k+1} \sum_{t=1}^{k} \frac{1}{t} \\
& \le \frac{2}{k+1} \left( 1 + \int_1^{k} \frac{1}{\xi} d\xi \right) = \frac{2}{k+1} \left( 1 + \log(k) \right).
\end{align*}
Then, by~\cref{aligned_19},
\begin{align}
\sum_{t=1}^{k} \frac{1}{t^2+t} \sum_{h=\mathrlap{k+1-t}}^{k} \alpha^{(h)} \E \left[ H(x^{(h)}) - H(x^{(k-t)}) \right] &\le \half \sum_{t=1}^{k} \frac{1}{t^2+t} \alpha^2 \left( \frac{t+1}{k+1-t} \right) C_1 \nonumber \\
&= \frac{\alpha^2 C_1}{2} \sum_{t=1}^{k} \frac{1}{t^2+t} \nonumber \\
&\le \frac{\alpha^2 C_1}{k+1} \left( 1 + \log(k) \right).
\label{aligned_20}
\end{align}
The logarithmic factor results from summing decreasing step sizes.

\emph{Step~4: Combine all bounds.}
Combine~\cref{aligned_15},~\cref{aligned_18} and~\cref{aligned_20} to obtain
\[
\alpha^{(k)} \E \bigl[ H(x^{(k)}) - H^* \bigr] \le \frac{1}{2(k+1)} \left( \Norm{x^{(0)} - x^*}^2 + \left(3 + 3\log(k+1) \right) \alpha^2 C_1 \right).
\]
Divide both sides by $\alpha^{(k)}$ and substitute $\alpha^{(k)} = \alpha/\sqrt{k+1}$
to yield~\cref{eq:main_theorem_1_eqiv}. Minimize over all $x^* \in \mathcal{X}^*$ to obtain~\cref{eq:main_theorem_1}, as required.
\end{proof}

\subsection{Convergence under $L$-Smoothness}
\label{sec:L_smooth}

We now analyze \BlockProx\ when $F$ is $L$-smooth but does not necessarily have bounded gradients, which is \cref{assumption:main} with Case~\ref{item:assum-loss}\ref{item:smooth-F}. In this setting, we establish convergence for the averaged iterate rather than the last iterate.

\subsubsection{Supporting Analysis}

The simple and important insight for the $L$-smooth case is that we can bound the squared gradient norm in terms of the objective gap.

\begin{lemma}[Gradient Bound via Smoothness]
\label{lem:gradient_bound_smooth}
Under \cref{assumption:main} with Case \ref{item:assum-loss}\ref{item:smooth-F} (smoothness of $F$), for any $x \in \R^{nd}$ and any minimizer $x^* \in \mathcal{X}^*$,
\[
\norm{\nabla F(x)} \le 2 \norm{\nabla F(x^*)} + 4L(H(x) - H^*).
\]
\end{lemma}

\begin{proof}
Observe that  
\begin{align*}
\|\nabla F(x) \|^2 &\le 2\| \nabla F(x^*)\|^2 + 2\| \nabla F(x) - \nabla F(x^*) \|^2\\
&\le 2\| \nabla F(x^*)\|^2 + 4L(F(x) - F(x^*) - \langle \nabla F(x^*), x - x^*\rangle)\\
&\le 2\| \nabla F(x^*)\|^2 + 4L(H(x) - H(x^*)),
\end{align*}
where the second inequality uses $L$-smoothness of $F$ and the third uses $-\nabla F(x^*) \in \partial G(x^*)$ from optimality of $x^*$.
\end{proof}

Applying \Cref{lem:gradient_bound_smooth}, we can get the following lemma for $L$-smooth $F$ similar to \Cref{lem:main_inequality_bounded_subg}.

\begin{lemma}[Expected Descent under Smoothness of $F$]
\label{lem:main_inequality_L_smooth}
Let \cref{assumption:main} hold with Case \ref{item:assum-loss}\ref{item:smooth-F} (smoothness of $F$).
Let $\alpha^{(k)} \le 1/(8L)$ for all $k$. Then \BlockProx\ generates a sequence $\{x^{(k)}\}$ that satisfies, for any $k$ and any $x^* \in \mathcal{X}^*$,
\[
\E^k \norm{x^{(k+1)} - x^*}^2 \le \norm{x^{(k)} - x^*}^2 - \alpha^{(k)} [ H(x^{(k)}) - H(x^*) ] + 7m^2 c^2 (\alpha^{(k)})^2.
\]
\end{lemma}

\begin{proof}
Since $F$ is assumed to be smooth in this case, we have $g^{(k)} = \nabla F(x^{(k)})$ and \cref{lem:proximal_step} gives
\begin{align}
\E^k \norm{x^{(k+1)} - x^*}^2 &\le \norm{z^{(k)} - x^*}^2 - 2 \alpha^{(k)} \big[ G(x^{(k)}) - G(x^*) \big] \nonumber \\
&\qquad + 2 mc (\alpha^{(k)})^2 \norm{\nabla F(x^{(k)})} + 2 m^2 c^2 (\alpha^{(k)})^2 \nonumber \\
&\le \norm{z^{(k)} - x^*}^2 - 2 \alpha^{(k)} \big[ G(x^{(k)}) - G(x^*) \big] \label{aligned_1_lem:main_inequality_L_smooth} \\
&\qquad + (\alpha^{(k)})^2 \norm{\nabla F(x^{(k)})}^2 + 3 m^2 c^2 (\alpha^{(k)})^2, \nonumber
\end{align}
where the last equality uses AM-GM inequality on the cross term
\[
2 m c \norm{\nabla F(x^{(k)})} \le \norm{\nabla F(x^{(k)})}^2 + m^2 c^2.
\]
For the gradient step $z^{(k)} = x^{(k)} - \alpha^{(k)} \nabla F(x^{(k)})$, by convexity of $F$, we have
\[
\norm{z^{(k)} - x^*}^2 \le \norm{x^{(k)} - x^*}^2 - 2 \alpha^{(k)} \big[ F(x^{(k)}) - F(x^*) \big] + (\alpha^{(k)})^2 \norm{\nabla F(x^{(k)})}^2.
\]
Combining this bound with~\cref{aligned_1_lem:main_inequality_L_smooth}, we obtain
\begin{align}
\E^k \norm{x^{(k+1)} - x^*}^2 &\le \norm{x^{(k)} - x^*}^2 - 2 \alpha^{(k)} \big[ H(x^{(k)}) - H(x^*) \big] \label{aligned_2_lem:main_inequality_L_smooth} \\
&\qquad + 2 (\alpha^{(k)})^2 \norm{\nabla F(x^{(k)})}^2 + 3 m^2 c^2 (\alpha^{(k)})^2.\nonumber
\end{align}
\cref{lem:gradient_bound_smooth} gives
\begin{align}
\norm{\nabla F(x^{(k)})}^2 &\le 2 \norm{\nabla F(x^*)}^2 + 4L(H(x^{(k)}) - H(x^*)) \nonumber \\
&\le 2 (mc)^2 + 4L(H(x^{(k)}) - H(x^*)), \label{aligned_3_lem:main_inequality_L_smooth}
\end{align}
where the last inequality uses $-\nabla F(x^*) \in \partial G(x^*) = \tsum_{j=1}^m \partial G_j(x^*)$ and \cref{assumption:main}\ref{item:assum-reg}.
Combining~\cref{aligned_2_lem:main_inequality_L_smooth} and~\cref{aligned_3_lem:main_inequality_L_smooth}, we have
\begin{align*}
\E^k \norm{x^{(k+1)} - x^*}^2 &\le \norm{x^{(k)} - x^*}^2 - (2 - 8L (\alpha^{(k)})) \alpha^{(k)} \big[ H(x^{(k)}) - H(x^*) \big] \\
&\qquad + 7 m^2 c^2 (\alpha^{(k)})^2.
\end{align*}
Finally, since $\alpha^{(k)} \le 1/(8L)$, we have $2 - 8L (\alpha^{(k)}) \ge 1$, which completes the proof.
\end{proof}

\subsubsection{Main Result}

\begin{theorem}[Convergence under Smoothness of $F$]
\label{thm:L_smooth_convergence}
Let \cref{assumption:main} hold with Case~\ref{item:assum-loss}\ref{item:smooth-F} (smoothness of $F$). Let $\alpha^{(k)} = \alpha/\sqrt{k+1}$ with $0 < \alpha \le \frac{1}{8L}$. Then \BlockProx\ generates a sequence $\{x^{(k)}\}$ that satisfies, for any $k$,
$$
\E [H(\Bar{x}^{(k)}) - H^*] \le \frac{1}{\sqrt{k+1}} \left( \frac{1}{\alpha} \dist(x^{(0)}, \mathcal{X}^*)^2 + \BigO(\alpha c^2 m^2 \log k) \right),
$$
where $\Bar{x}^{k} = 1/(k+1) \sum_{t=0}^k x^{(t)}$, $\dist(x^{(0)}, \mathcal{X}^*) = \min_{x^* \in \mathcal{X}^*} \Norm{x^{(0)} - x^*}$, and the term $\BigO(\alpha c^2 m^2 \log k)$ can be made explicit as $7m^2 c^2 \alpha \left( 1 + \log(k+1) \right)$.
\end{theorem}

\begin{proof}
From \cref{lem:main_inequality_L_smooth}, we have:
$$
\E^t \norm{x^{(t+1)} - x^*}^2  \le \norm{x^{(t)} - x^*}^2 - \alpha^{(t)} [ H(x^{(t)}) - H(x^*) ] + 7m^2 c^2 (\alpha^{(t)})^2.
$$
Taking full expectation and summing from $t=0$ to $k$:
\begin{align}
\sum_{t=0}^{k} \alpha^{(t)}\E[H(x^{(t)}) - H(x^*)] &\le \| x^{(0)} - x^*\|^2 + 7m^2 c^2 \sum_{t=0}^{k} (\alpha^{(t)})^2 \nonumber \\
&\le \| x^{(0)} - x^*\|^2 + 7m^2 c^2 \alpha^2 \left( 1 + \log(k+1) \right). \label{aligned_1_thm:L_smooth_convergence}
\end{align}

Since $\alpha^{(t)} = \alpha/\sqrt{t+1}$ is decreasing and $H(x^{(t)}) - H(x^*) \ge 0$, by Jensen's inequality, we have:
\begin{align*}
\sum_{t=0}^{k} \alpha^{(t)}\E[H(x^{(t)}) - H(x^*)] &\ge \sum_{t=0}^{k} \alpha^{(k)}\E[H(x^{(t)}) - H(x^*)] \\
&\ge \alpha \sqrt{k+1} \E [H(\Bar{x}^{(t)}) - H(x^*)]
\end{align*}
Combining this with \cref{aligned_1_thm:L_smooth_convergence}, dividing by $\alpha \sqrt{k+1}$, and taking infimum over all $x^* \in \mathcal{X}^*$, we obtain the desired result.
\end{proof}

\begin{remark}
Unlike the bounded subgradient case (\cref{thm:main_theorem_1}), this result provides convergence only for the averaged iterate. This reflects a fundamental difference in the analysis: when subgradients are unbounded, our proof technique cannot establish a last-iterate convergence without additional assumptions. On the other hand, many applications such as the ones in our experiments do not satisfy the bounded subgradient assumption, yet they satisfy $L$-smoothness. The averaging serves as an analytical tool to handle this case. 
\end{remark}

\subsection{Convergence under Strong Convexity}
\label{sec:strongly_convex}

We next analyze \BlockProx\ under strong convexity, where both $F$ and the composite objective $H$ inherit strong convexity.

\subsubsection{Supporting Lemma}

\begin{lemma}[Expected Descent under Smoothness and Strong Convexity of $F$]
\label{lem:main_inequality_strong_convex}
Let \cref{assumption:main} hold with Case~\ref{item:assum-loss}\ref{item:strong-convex-F} (smoothness and strong convexity of $F$). Then \BlockProx\ generates a sequence $\{x^{(k)}\}$ that satisfies, for any $k$ and any $y \in \R^{nd}$,
\begin{align*}
\E^k \norm{x^{(k+1)} - y}^2
&\le \left(1 - \alpha^{(k)} \mu + 3 (\alpha^{(k)})^2 L^2\right) \|x^{(k)} - y\|^2
\\ &\qquad - 2 \alpha^{(k)} (H(x^{(k)}) - H(y)) 
\\ &\qquad + \left(3m^2 c^2 + 2m \|\nabla F(y)\| c + 2 \|\nabla F(y)\|^2 \right) (\alpha^{(k)} )^2.
\end{align*}
\end{lemma}

\begin{proof}
Use the $L$-smoothness of $F$ to obtain the bound on the gradient:
\begin{equation}
\label{aligned_1_lem:main_inequality_strong_convex}
\|\nabla F(x^{(k)})\| \le \|\nabla F(x^{(k)}) - \nabla F(y)\| + \|\nabla F(y)\| \le L \|x^{(k)} - y\| + \|\nabla F(y)\|, 
\end{equation}
and the squared gradient bound
\begin{equation}
\label{aligned_2_lem:main_inequality_strong_convex}
\begin{aligned}
\|\nabla F(x^{(k)})\|^2 &= \|\nabla F(x^{(k)}) - \nabla F(y) + \nabla F(y)\|^2 \\
&\le 2 \|\nabla F(x^{(k)}) - \nabla F(y)\|^2 + 2 \|\nabla F(y)\|^2 \\
&\le 2 L^2 \|x^{(k)} - y\|^2 + 2 \|\nabla F(y)\|^2. 
\end{aligned}
\end{equation}
Since $F$ is differentiable, $g^{(k)} = \nabla F(x^{(k)})$ and \cref{lem:proximal_step} gives
\begin{align}
\E^k \norm{x^{(k+1)} - y}^2 &\le \norm{z^{(k)} - y}^2 - 2 \alpha^{(k)} \big[ G(x^{(k)}) - G(y) \big] \nonumber \\
&\qquad + 2 mc (\alpha^{(k)})^2 \norm{\nabla F(x^{(k)})} + 2 m^2 c^2 (\alpha^{(k)})^2 \nonumber \\
&\le \norm{z^{(k)} - y}^2 - 2 \alpha^{(k)} \big[ G(x^{(k)}) - G(y) \big] + 2 m^2 c^2 (\alpha^{(k)})^2 \label{aligned_3_lem:main_inequality_strong_convex} \\
&\qquad + 2 mc (\alpha^{(k)})^2 (L \|x^{(k)} - y\| + \|\nabla F(y)\|). \nonumber
\end{align}
For the gradient step, expand $z^{(k)} = x^{(k)} - \alpha^{(k)} \nabla F(x^{(k)})$ and use $\mu$-strong convexity:
\begin{equation}\label{aligned_12}
\begin{aligned}
\|z^{(k)} - y\|^2 &= \|x^{(k)} - y\|^2 - 2 \alpha^{(k)} \langle\nabla F(x^{(k)}), x^{(k)} - y\rangle + (\alpha^{(k)})^2 \|\nabla F(x^{(k)})\|^2 \\
&\le \|x^{(k)} - y\|^2 - 2 \alpha^{(k)} \left( F(x^{(k)}) - F(y) + \frac{\mu}{2} \|x^{(k)} - y\|^2 \right) \\
& \qquad + (\alpha^{(k)})^2 \left( 2L^2 \|x^{(k)} - y\|^2 + 2\|\nabla F(y)\|^2 \right) \\
&= \left( 1 - \alpha^{(k)} \mu + 2 (\alpha^{(k)})^2 L^2 \right) \|x^{(k)} - y\|^2 - 2 \alpha^{(k)} ( F(x^{(k)}) - F(y) ) \\
& \qquad + 2 (\alpha^{(k)})^2 \|\nabla F(y)\|^2.
\end{aligned}
\end{equation}

Combining these bounds with $\beta^{(k)} = m \alpha^{(k)}$ and applying the AM-GM inequality to the cross term $2 m (\alpha^{(k)} )^2 Lc \|x^{(k)} - y\|$ yields the result.
\end{proof}

\subsubsection{Main Result}
Then, using \Cref{lem:main_inequality_strong_convex}, we can establish the following convergence result, similar to standard SGD (see, e.g., \citep[Chapter 5.4.3]{Wright_Recht_2022}).
\begin{theorem}[Convergence under Smoothness and Strong Convexity of $F$]
\label{thm:sublinear-rate-convergence-strong-convexity}
Let \cref{assumption:main} hold with Case~\ref{item:assum-loss}\ref{item:strong-convex-F} (smoothness and strong convexity of $F$). Because $H$ is strongly convex, it has a unique minimizer $x^{*}$. Let
\begin{equation}
\label{eq:main_theorem_2_stepsize}
\alpha^{(k)} = \frac{2 \mu}{\mu^2 k + 12L^2}.
\end{equation}
Then \BlockProx\ generates a sequence $\{x^{(k)}\}$ that satisfies, for all $k$,
\begin{equation}
\label{eq:main_theorem_2}
\E \Norm{x^{(k)} - x^*}^2 \le \frac{A}{\mu^2 k + 12L^2},
\end{equation}
for constant $A$ that does not depend on $k$. Specifically,
$$
A \coloneqq \max \left\{ 12 L^2 \left( \Norm{x^{(0)} - x^*}^2 \right), \left( \mu^2 + 12L^2 \right) \left( \E \Norm{x^{(1)} - x^*}^2 \right), 4 \left( 1 + \frac{12 L^2}{\mu^2} \right)C_2 \right\},
$$
where $C_2 = 7m^2 c^2$.
\end{theorem}

\begin{proof}
We prove~\cref{eq:main_theorem_2} by induction.

\emph{Step 1: Base cases and setup.}
The base case $k = 0$ and $k=1$ is given directly by
$$ A \ge 12 L^2 \left( \Norm{x^{(0)} - x^*}^2 \right), \quad A \ge \left( \mu^2 + 12L^2 \right) \left( \E \Norm{x^{(1)} - x^*}^2 \right) $$
Assume that~\cref{eq:main_theorem_2} holds for $k \ge 1$, i.e.,
\begin{equation}
\label{eq:main_theorem_2_hypothesis}
\E \Norm{x^{(k)} - x^*}^2 \le \frac{A}{\mu^2 k + 12L^2}.
\end{equation}
For convenience, we denote $B^{(k)} \coloneqq \E^k \Norm{x^{(k)} - x^*}^2$ and $Q_k \coloneqq \mu^2 k + 12L^2$ in this proof.
Then, the induction hypothesis~\cref{eq:main_theorem_2_hypothesis} and the stepsize~\cref{eq:main_theorem_2_stepsize} becomes
\begin{equation}
\label{aligned_00}
B^{(k)} \le \frac{A}{Q_k}, \quad \alpha^{(k)} = \frac{2 \mu}{Q_k}.
\end{equation}

\emph{Step 2: Inductive step via the main inequality.}
Since $x^*$ is the unique minimizer of $H$, by optimality conditions, we have
\[
0 \in \partial H(x^*) = \nabla F(x^*) + \sum_{j=1}^m \partial G_j(x^*),
\]
which implies there exists $g_j^* \in \partial G_j(x^*)$ for $j = 1, \ldots, m$ such that $\nabla F(x^*) = - \sum_{j=1}^m g_j^*$. Then by \cref{assumption:main}(b) and triangle inequality, we have
\begin{equation}
\label{grad_F_opt_bound}
\|\nabla F(x^*)\| \le \sum_{j=1}^m \| g_j^* \| \le mc.
\end{equation}
Now we can set $y = x^{*}$ in \cref{lem:main_inequality_strong_convex} and by~\cref{grad_F_opt_bound}, we have
\begin{align*}
\E^k\norm{x^{(k+1)} - x^*}^2 &\le \left(1 - \alpha^{(k)} \mu + 3 (\alpha^{(k)})^2 L^2\right) \|x^{(k)} - x^*\|^2 \\ 
&\qquad - 2 \alpha^{(k)} (H(x^{(k)}) - H(x^*)) + 7m^2 c^2 (\alpha^{(k)} )^2.
\end{align*}
Take full expectation to obtain a bound on $B^{(k+1)} \coloneqq \E \|x^{(k+1)} - x^*\|^2$:
\begin{align*}
B^{(k+1)} &\le \left( 1 - \alpha^{(k)} \mu + 3 \left(\alpha^{(k)}\right)^2 L^2 \right) B^{(k)} - 2 \alpha^{(k)} \left( \E \left[H(x^{(k)}) - H(x^{*})\right] \right) + \bigl(\alpha^{(k)}\bigr)^2 C_2 \\
&\le \left( 1 - \alpha^{(k)} \mu + 3 \left(\alpha^{(k)}\right)^2 L^2 \right) B^{(k)} +  \big(\alpha^{(k)}\big)^2 C_2.
\end{align*}
Then, by~\cref{aligned_00},
\[
\begin{aligned}
  B^{(k+1)} & \le \left( 1 - \left( \frac{2 \mu}{Q_k} \right) \mu + 3  \left( \frac{2 \mu}{Q_k} \right)^2 L^2 \right) \frac{A}{Q_k} + \left(\frac{2 \mu}{Q_k}\right)^2 C_2 \\
           & = \frac{Q_k^2 - 2Q_k\mu ^2 + 12\mu ^2L^2}{Q_k^3} A + \frac{4\mu^2}{Q_k^2}C_2.
\end{aligned}
\]

\emph{Step 3: Establishing the contraction.}
Multiplying the both sides of the above equation by \( (Q_k + \mu^2)Q_k^3 \), we have
\begin{align}
\left( Q_k + \mu^2 \right) Q_k^3 B^{(k+1)} &\le \left( Q_k + \mu^2 \right) \left( Q_k^2 - 2 \mu^2 Q_k + 12 \mu^2 L^2 \right) A + 4 \mu^2 Q_k \left( Q_k + \mu^2 \right) C_2 \nonumber \\
&= Q_k^3 A - \left( \mu^2 Q_k^2 + 2 \mu^4 Q_k - 12 \mu^2 L^2 Q_k - 12 \mu^4 L^2 \right) A \nonumber\\
&\quad \quad + 4 \mu^2 Q_k \left( Q_k + \mu^2 \right) C_2. \label{aligned_02}
\end{align}

\emph{Step 4: Verifying the key inequality.}
By the definition of $A$,
\begin{align*}
\left( \mu^2 Q_k^2 + 2 \mu^4 Q_k - 12 \mu^2 L^2 Q_k - 12 \mu^4 L^2 \right) A &> \left( \mu^2 Q_k^2 + \mu^4 Q_k - 12 \mu^2 L^2 Q_k - 12 \mu^4 L^2 \right) A \\
&= \mu^2 \left( Q_k + \mu^2 \right) \left( Q_k - 12L^2 \right) A, 
\end{align*}
and
\begin{align*}
\left( Q_k - 12L^2 \right) A &= \mu^2 k A \ge \mu^2 k \cdot 4 \left( 1+ \frac{12L^2}{\mu^2} \right)C_2 \\
&= 4 \left( \mu^2 k + 12 L^2 k \right) C_2 \\
&\ge 4 \left( \mu^2 k + 12 L^2 \right) C_2 = 4 Q_k C_2,
\end{align*}
where the second inequality holds because $k \ge 1$. Thus,
\begin{equation}
\label{aligned_03}
\left( \mu^2 Q_k^2 + 2 \mu^4 Q_k - 12 \mu^2 L^2 Q_k - 12 \mu^4 L^2 \right) A  \ge 4 \mu^2 Q_k \left( Q_k + \mu^2 \right) C_2.
\end{equation}

\emph{Step 5: Completing the induction.}
Combining \cref{aligned_02} and \cref{aligned_03}, we have
\[
\left( Q_k + \mu^2 \right) Q_k^3 B^{(k+1)} \le Q_k^3 A,
\]
which is equivalent to
\[
B^{(k+1)} \le \frac{A}{Q_k + \mu^2} = \frac{A}{\mu^2 (k+1) + 12L^2},
\]
which completes the proof.
\end{proof}

\subsection{Linear Convergence to a Neighborhood}
\label{sec:linear_convergence}

This next result establishes linear convergence to a neighborhood of the solution under the assumption of smoothness and strong convexity of $F$. The radius of the neighborhood is proportional to $c\sqrt{m\alpha/\mu}$, where $c$ bounds the regularizer subgradients $\partial G_j$, $m$ is the number of regularizer components $j=1,\ldots,m$, $\alpha$ is the stepsize, and $\mu$ is the strong convexity parameter. The neighborhood results from the algorithm's randomized local coordination mechanism: each node samples only one regularizer component per iteration, which creates an approximation error that scales with the stepsize and the number of components.

\begin{theorem}[Linear Convergence to a Neighborhood]
\label{thm:linear-rate-neighborhood}
Let \cref{assumption:main} hold with Case~\ref{item:assum-loss}\ref{item:strong-convex-F} (smoothness and strong convexity of $F$). Because $H$ is strongly convex, it has a unique minimizer $x^{*}$.
Let $\alpha^{(k)}\equiv\alpha$ be a constant with $0 < \alpha < (2\mu) / (3L^2)$. Then \BlockProx\ generates a sequence $\{x^{(k)}\}$ that satisfies, for any $k$,
\begin{equation}\label{eq:dist_linear}
\E \Norm{x^{(k)}-x^*}^{2} \le (1 - \Gamma)^k \Norm{x^{(0)}-x^*}^{2} + \frac{\delta}{\Gamma}
\end{equation}
where
\begin{equation} \label{eq:gamma-delta-in-theorem-linear-rate-neighborhood}
  \Gamma \coloneqq 2\alpha\mu-3\alpha^{2}L^{2} \in (0,1),
  \qquad
  \delta \coloneqq 7 m^2 c^2 \alpha^2 = \BigO\big((m\alpha c)^2\big).
\end{equation}
Moreover, the objective gap obeys the bound
\begin{equation}\label{eq:obj_linear}
   \E[H(x^{(k)})-H(x^*)]\le
   \Bigl(\tfrac{L}{2}\|x^{(0)}-x^*\|^{2}+2mc\|x^{(0)}-x^*\|\Bigr)(1-\Gamma)^{k/2}
    + \frac{L\delta}{2\Gamma}+2mc \sqrt{\frac{\delta}{\Gamma}}.
  \end{equation}
\end{theorem}

\begin{remark}
    The distance recursion~\cref{eq:dist_linear} yields a linear (geometric) rate, which is similar to standard convergence rate of stochastic gradient descent method with constant stepsize for strongly convex and $L$-smooth objective functions~\citep[Theorem 5.8]{GG2023}.  The objective gap inherits the same rate up to a square-root factor that is immaterial for geometric convergence. The residual term $2mc\sqrt{\delta/\Gamma}$ reflects the first-order approximation error from the non-smooth regularizer near optimality. 
\end{remark}

\begin{proof}
Following the analysis in \cref{thm:sublinear-rate-convergence-strong-convexity}, we have $\| \nabla F(x^*) \| \le mc$ as~\cref{grad_F_opt_bound}. 
With $y=x^*$ in \cref{lem:main_inequality_strong_convex} and taking expectation we obtain
\begin{align*}
\E \norm{x^{(k+1)} - x^*}^2 &\le \left(1 - \alpha \mu + 3 \alpha^2 L^2 \right) \E \norm{x^{(k)} - x^*}^2 \\
&\qquad - 2 \alpha \E [H(x^{(k)}) - H(x^*)] + 7 m^2 c^2 \alpha^2 \\
&\le \left(1 - 2\alpha \mu + 3 \alpha^2 L^2 \right) \E \norm{x^{(k)} - x^*}^2 + 7 m^2 c^2 \alpha^2,
\end{align*}
where the second inequality uses the strong convexity of $H$ and the optimality of $x^*$. The above inequality can be written as
\begin{equation}
\label{eq:dist_linear_1}
\E \norm{x^{(k+1)} - x^*}^2 \le (1 - \Gamma) \E \norm{x^{(k)} - x^*}^2 + \delta,
\end{equation}
where $\Gamma$ and $\delta$ are defined in~\cref{eq:gamma-delta-in-theorem-linear-rate-neighborhood}.
Here, since $0 < \alpha < (2\mu) / (3L^2)$, we have 
$$ \Gamma = 2\alpha\mu - 3\alpha^2 L^2 = \frac{\mu^2}{3L^2} - 3 L^2 \left( \alpha - \frac{\mu}{3L^2} \right)^2 \in (0, \mu^2 / (3L^2) ] \subseteq (0, 1). $$
Unrolling the recursion~\cref{eq:dist_linear_1} gives~\cref{eq:dist_linear}.

Then, use the $L$-smoothness of $F$ and the $c$-Lipschitz continuity of $G_j$:
\begin{align*}
  F(x^{(k)})-F(x^*) &\le \langle\nabla F(x^*),x^{(k)}-x^*\rangle+\tfrac{L}{2}\|x^{(k)}-x^*\|^{2} \\
  G(x^{(k)})-G(x^*) &\le mc\,\|x^{(k)}-x^*\|.
\end{align*}
Because $H=F+G$,
\begin{align*}
  H(x^{(k)}) - H(x^*) &= [F(x^{(k)}) - F(x^*)] + [G(x^{(k)}) - G(x^*)] \nonumber\\
  &\le \langle\nabla F(x^*), x^{(k)}-x^*\rangle + \tfrac{L}{2}\|x^{(k)}-x^*\|^2 + mc \|x^{(k)}-x^*\| \nonumber\\
  &\le \|\nabla F(x^*)\| \cdot \|x^{(k)}-x^*\| + \tfrac{L}{2}\|x^{(k)}-x^*\|^2 + c\|x^{(k)}-x^*\| \nonumber\\
  &\le 2mc \|x^{(k)}-x^*\| + \tfrac{L}{2}\|x^{(k)}-x^*\|^2,
\end{align*}
where the second inequality uses Cauchy-Schwarz inequality and the last inequality uses $\|\nabla F(x^*)\| \le mc$ as in \cref{grad_F_opt_bound}.
Taking expectations gives that for any $k$,
\begin{align}
\E [H(x^{(k)}) - H(x^*)] &\le \tfrac{L}{2} \E\|x^{(k)}-x^*\|^2 + 2mc \E \|x^{(k)}-x^*\| \nonumber \\
&\le \tfrac{L}{2} \E \|x^{(k)}-x^*\|^2 + 2mc \sqrt{\E\|x^{(k)}-x^*\|^2}, \label{eq:H_bound}
\end{align}
where the second inequality uses Jensen's inequality for the concave square-root function.
Insert the bound~\cref{eq:dist_linear} into \cref{eq:H_bound}, and use the inequality $\sqrt{a+b}\le\sqrt{a}+\sqrt{b}$ to obtain~\cref{eq:obj_linear}.
\end{proof}

The preceding theorem characterizes convergence for arbitrary partially separable regularizers. We now specialize to graph-guided regularizers, whose structure yields sharper bounds. Each regularizer component $G_j$ corresponds to a graph edge, giving $m = |\cE|$. The \RandomEdge\ algorithm inherits linear convergence while using only two messages per iteration in expectation of the graph size $n$ or topology. The following corollary shows how graph topology determines the convergence neighborhood.

\begin{corollary}[Linear Convergence for \RandomEdge]
\label{cor:randomedge_linear}
Consider the graph-guided regularization problem~\cref{eq:multi_task_with_graph} on a communication graph $\cG = (\cV, \cE)$ with $n$ nodes and $m = |\cE|$ edges. Under the conditions of \cref{thm:linear-rate-neighborhood}, the \RandomEdge\ algorithm achieves
$$
  \E \|x^{(k)}-x^*\|^{2}\le (1-\Gamma)^{k}\,\norm{x^{(0)}-x^*}^{2}+\BigO(\alpha c^2|\cE|^2).
$$
For large graphs with $|\cE| \gg 1$ and small step sizes satisfying $\alpha L^2 \ll \mu$, the convergence rate simplifies to $\Gamma \approx 2\alpha\mu$.
\end{corollary}

\begin{proof}
By \cref{prop:equivalence}, \RandomEdge\ is equivalent to \BlockProx\ for graph-guided regularizers where $m = |\cE|$.  Substitute this expression into the general result~\cref{eq:gamma-delta-in-theorem-linear-rate-neighborhood} gives $\delta = \BigO\big((m\alpha c)^{2}\big) = \BigO\big((|\cE|\alpha c)^2\big)$.
When $|\cE| \gg 1$ and $\alpha L^2 \ll \mu$, the convergence rate simplifies to $\Gamma \approx 2\alpha\mu$ and the asymptotic expected squared distance is $\BigO\big(\alpha c^{2}|\cE|^2\big)$.
\end{proof}

\begin{remark}
The corollary exposes a trade-off: \RandomEdge\ uses constant expected communication per iteration, but the asymptotic expected squared distance scales with $|\cE|^2$. Sparse graphs ($|\cE| = \BigO(n)$) yield squared distances of order $\BigO(\alpha c^2 n^2)$, while dense graphs ($|\cE| = \BigO(n^2)$) give $\BigO(\alpha c^2 n^4)$. More edges mean larger convergence neighborhoods through which to randomly coordinate locally.
\end{remark}

\section{Numerical Experiments}
\label{sec:numerical_experiments}

In this section, we evaluate our algorithms through numerical experiments on synthetic and real-world datasets with graph-guided regularizers. We compare \RandomEdge\ against four established methods:
\begin{itemize}
    \item \textsc{ADMM} for network LASSO~\citep{HLB2015};
    \item \textsc{ProxAvg} (proximal averaging)~\citep{Y2013};
    \item \textsc{DSGD} (distributed subgradient descent)~\citep{nedic_distributed_2009};
    \item \textsc{Walkman}~\citep{MYHGSY2020}.
\end{itemize}
Each baseline algorithm has specific applicability constraints: \textsc{ADMM} for network LASSO only handles regularizers of the form $g_{ij}(x_i, x_j) = w_{ij} \norm{x_i - x_j}_2$, where $w_{ij} \geq 0$ is the weight of edge \( (i, j) \); \textsc{ProxAvg} supports general graph-guided regularizers as in~\cref{eq:graph_guided_regularizer}; while \textsc{DSGD} and \textsc{Walkman} are limited to consensus optimization problems.

To apply consensus optimization algorithms to multitask learning problems, every node in the graph needs to maintain copies of all variables across all nodes, and the regularizer is embedded in the local loss function.
This approach has two drawbacks: high \( \BigO(nd) \) memory complexity for each node, and the nonsmoothness of the local loss function when the regularizer is nonsmooth, which prevents many standard consensus optimization methods from being applied, such as Scaffnew~\citep{MMSR2022}, DGD~\citep{YLY2016} and ESDACD~\citep{HBM2019}.

The graph-guided multitask learning problem in~\cref{eq:multi_task_with_graph} with a communication network \( (\mathcal{V}, \mathcal{E}) \) can be equivalently written as the consensus optimization problem
\begin{equation}
  \label{eq:equ-cons-prob}
   \min_{x \in \R^{nd}} \textstyle \sum_{i \in \mathcal{V}}  \{ f_i(x) + \sum_{(i, j) \in \mathcal{E}} \half g_{ij}(x_i, x_j) \},
\end{equation}
where \( x = (x_1, \dots , x_n) \in \R^{nd} \) represents the concatenated variable containing copies of all variables across all nodes, and \( f_i(x) = f_i(x_i) \). Observe the factor \( \half \) before \( g_{ij} \) that occurs because each edge is counted twice in the two endpoints.
We can then apply \textsc{DSGD} to solve~\cref{eq:equ-cons-prob}.
\textsc{DSGD} requires a doubly stochastic \textit{mixing matrix} to perform the mixing step. We use the Metropolis-Hastings weight matrix~\citep{XBL2006} as the mixing matrix, which is common in the decentralized optimization literature~\citep{SLWY2015}.
We apply \textsc{Walkman} to solve the following regularized problem, which is equivalent to~\cref{eq:multi_task_with_graph}:
\[
  \min_{x \in \R^{nd}} \textstyle \tfrac{1}{|\mathcal{V}|} \sum_{i \in \mathcal{V}} f_i(x) + \tfrac{1}{|\mathcal{V}|} G(x),
\]
where \( f_i(x) = f_i(x_i) \).  Note that the problem is scaled by \( 1 / |\mathcal{V}| \) because in the objective function defined by~\citet{MYHGSY2020}, the local loss function is normalized by \( |\mathcal{V}| \). To ensure the optimal solution set is the same, we also scale the regularizer so that only the optimal value is scaled.
\textsc{Walkman} requires a transition probability matrix \( P \) to choose the next agent in each iteration. We follow \citet[Example~2]{MYHGSY2020} and choose \( P = A / d_{\max} \), where \( A \) is the adjacency matrix of the graph, \( d_{\max} \) is the maximum degree in the graph.

We focus on optimization and communication complexity, comparing objective function values achieved by different algorithms after a fixed number of \emph{communications}.
Let $H^*$ denote the optimal objective value obtained by centralized convex solvers. For each algorithm, we measure the optimality gap $H(x^{(t)}) - H^*$ after $t$ communications.
We define one \emph{communication} as a single \( \R^d \) variable transmission along edge $(i,j) \in \mathcal{E}$.
With $m = |\mathcal{E}|$ denoting the number of edges in the communication network, \cref{tab:communication_comparison} summarizes the communication cost per iteration for each algorithm.

All algorithms are implemented in Julia 1.11, with experiments conducted on a laptop with an Apple M4 Pro processor (14 cores) and 48GB RAM.

\subsection{Synthetic Data: Least Squares Benchmarks}
\label{sec:synthetic_benchmarks}

For the synthesized dataset, we consider the network LASSO settings as in~\cite{HLB2015}.
The communication network contains an unknown number of node groups, where nodes within each group share the same ground truth parameters.
The objective is to learn different models simultaneously without prior knowledge of group structure.

Given the number of groups and the number of nodes in each group, we generate a communication network \( (\cV, \cE) \) with \( \abs{\cV} = n \) and \( \abs{\cE} = m \), where the nodes in the same group are connected with probability \( 0.5 \) while the nodes in different groups are connected with probability \( 0.01 \).
For node \( i \), the data \( A_i \in \R^{15 \times 21} \) is generated from a standard normal distribution with a bias term, and the output is computed by \( b_i = A_ix_i^{*} + \epsilon_i \in \R^{15} \), where \( x_i^{*} \in \R^{21} \) is the ground truth of \( i \)-th node and \( \epsilon_i \in \R^{15} \) is small white noise.
Then, the optimization problem is
\begin{equation}
  \label{eq:2-norm-LS}
  \min_{x_1, \dots , x_n \in \R^{21}} \textstyle \{ \sum_{i=1}^n \frac{1}{2} \Norm{A_ix_i - b_i}_2^2 + \lambda \sum_{(i, j) \in \cE} \Norm{x_i - x_j}_2 \}.
\end{equation}
Moreover, to show that our algorithm can handle different regularizers, we also consider the following optimization problem with a different regularizer
\begin{equation}
  \label{eq:1-norm-LS}
  \min_{x_1, \dots , x_n \in \R^{21}} \textstyle \{ \sum_{i=1}^n \frac{1}{2} \Norm{A_ix_i - b_i}_2^2 + \lambda \sum_{(i, j) \in \cE} \Norm{x_i - x_j}_1 \}.
\end{equation}
Note that these two problems are not strongly convex and the regularizers are nonseparable and nonsmooth. Therefore, the methods LoCoDL \citep{CMR2025} and CoCoA \citep{SFMTJJ2018} are not applicable here.

We consider three communication network settings:
\begin{enumerate}[label=(\roman*)]
  \item \( 5 \) groups with \( 10, 17, 18, 18, 12 \) nodes, respectively;
  \item \( 1 \) group with \( 20 \) nodes;
  \item \( 1 \) group with \( 40 \) nodes, where the communication network is fully connected (i.e., every pair of nodes is connected by an edge).
\end{enumerate}
For each setting, we fix \( 10{,}000 \) communications and perform \( 100 \) experiments to solve~\cref{eq:2-norm-LS} and~\cref{eq:1-norm-LS} with \( \lambda = 1 \).

We slightly tune the stepsizes to ensure all algorithms perform well.
Specifically, our algorithm uses \( \alpha^{(t)} = \tfrac{0.01}{\sqrt{t+1}} \) and \( \beta^{(t)} = \alpha^{(t)} m \); \textsc{ADMM} uses \( \rho = 10^{-4} + \sqrt{\lambda / 2} \) as suggested in~\cite{HLB2015}; \textsc{ProxAvg} and \textsc{DSGD} use stepsize \( 10^{-2} \); \textsc{Walkman} uses \( \beta = 10{,}000 \).
The results based on the \( 100 \) experiments are shown in \cref{fig:2-norm-LS,fig:1-norm-LS}, where the plotted curve represents the mean convergence error over the \( 100 \) random problem instances, with error bars denoting \( \pm 1 \) standard deviation at each iteration.
As shown, our algorithm (blue solid line) converges to the optimal solution much faster than other algorithms at the same communication cost.
Because the per-iteration communication cost is high, \textsc{ADMM}, \textsc{ProxAvg}, and \textsc{DSGD} perform only a few iterations and remain far from the optimal solution.
Although \textsc{Walkman} requires only \( n \) communications per iteration by sending an \( \R^{nd} \) variable via an edge, it updates only one node per iteration, which causes the objective function value to decrease slowly.
Moreover, the mean communications per iteration for \BlockProx\ across the 100 experiments for the three settings and two regularizers (1-norm and 2-norm) are \( 1.999532 \), \( 1.994982 \); \( 1.995706 \), \( 2.003392 \); \( 1.997192 \), \( 2.00109 \), respectively, which matches our communication complexity in~\cref{cor:randedge_communication}.

\begin{figure*}[t]
	\centering
	\subfloat[5 groups (10, 17, 18, 18, 12 nodes)]{\label{fig:2-norm-LS-5}
		\includegraphics[width=0.3\linewidth]{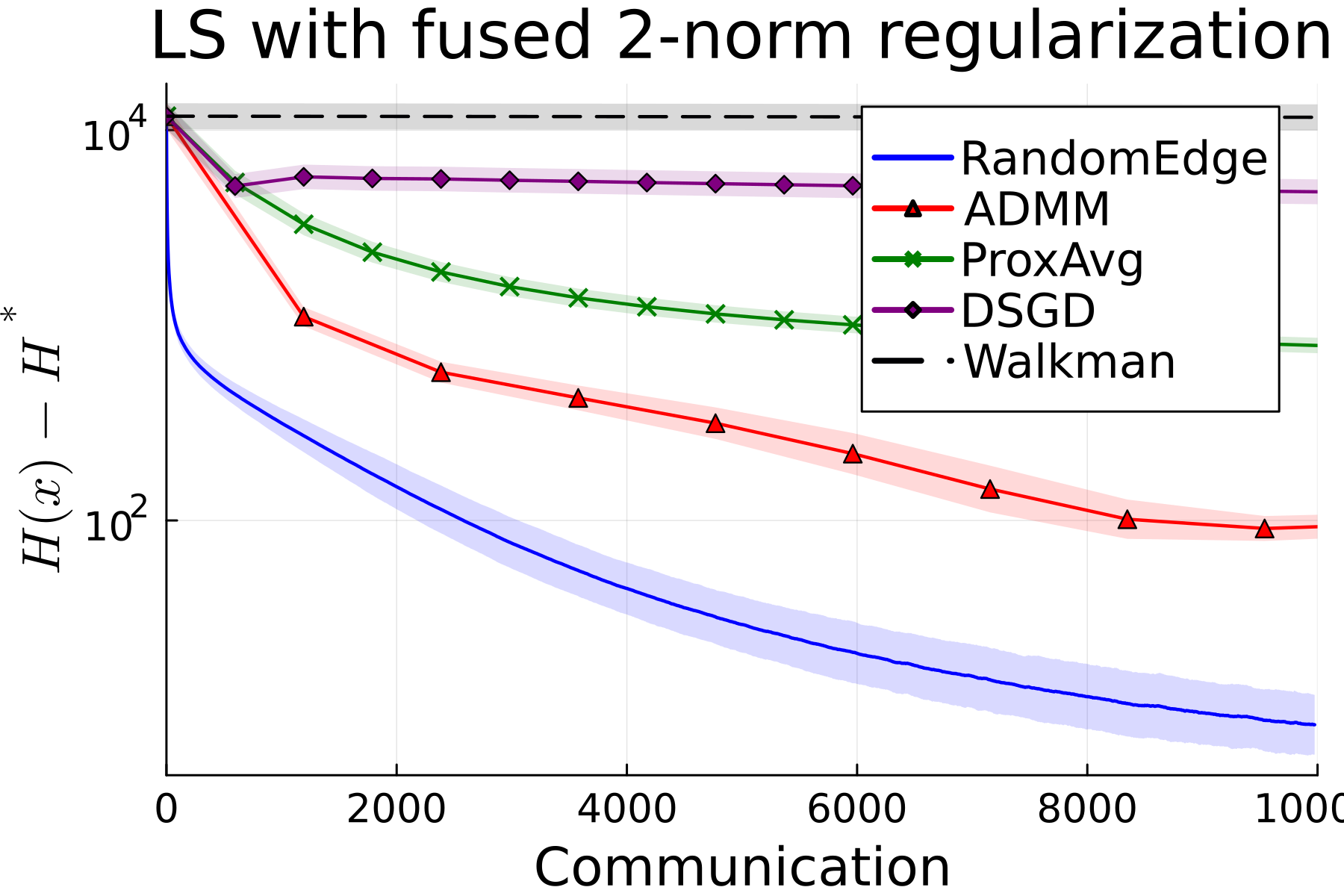}
	}
	\subfloat[1 group (20 nodes)]{\label{fig:2-norm-LS-1}
		\includegraphics[width=0.3\linewidth]{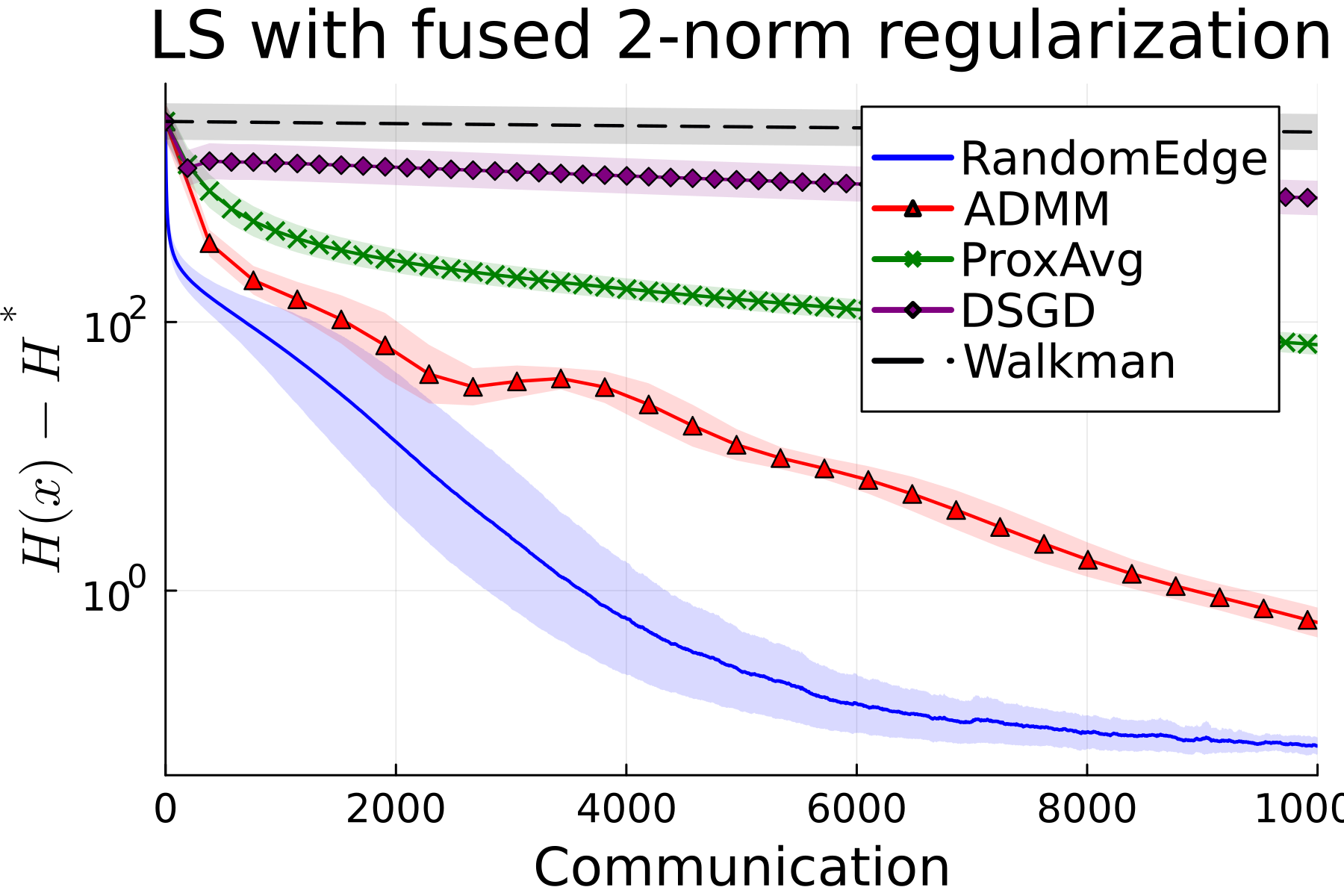}
	}
	\subfloat[1 group (40 nodes, complete graph)]{\label{fig:2-norm-LS-1-fully}
		\includegraphics[width=0.3\linewidth]{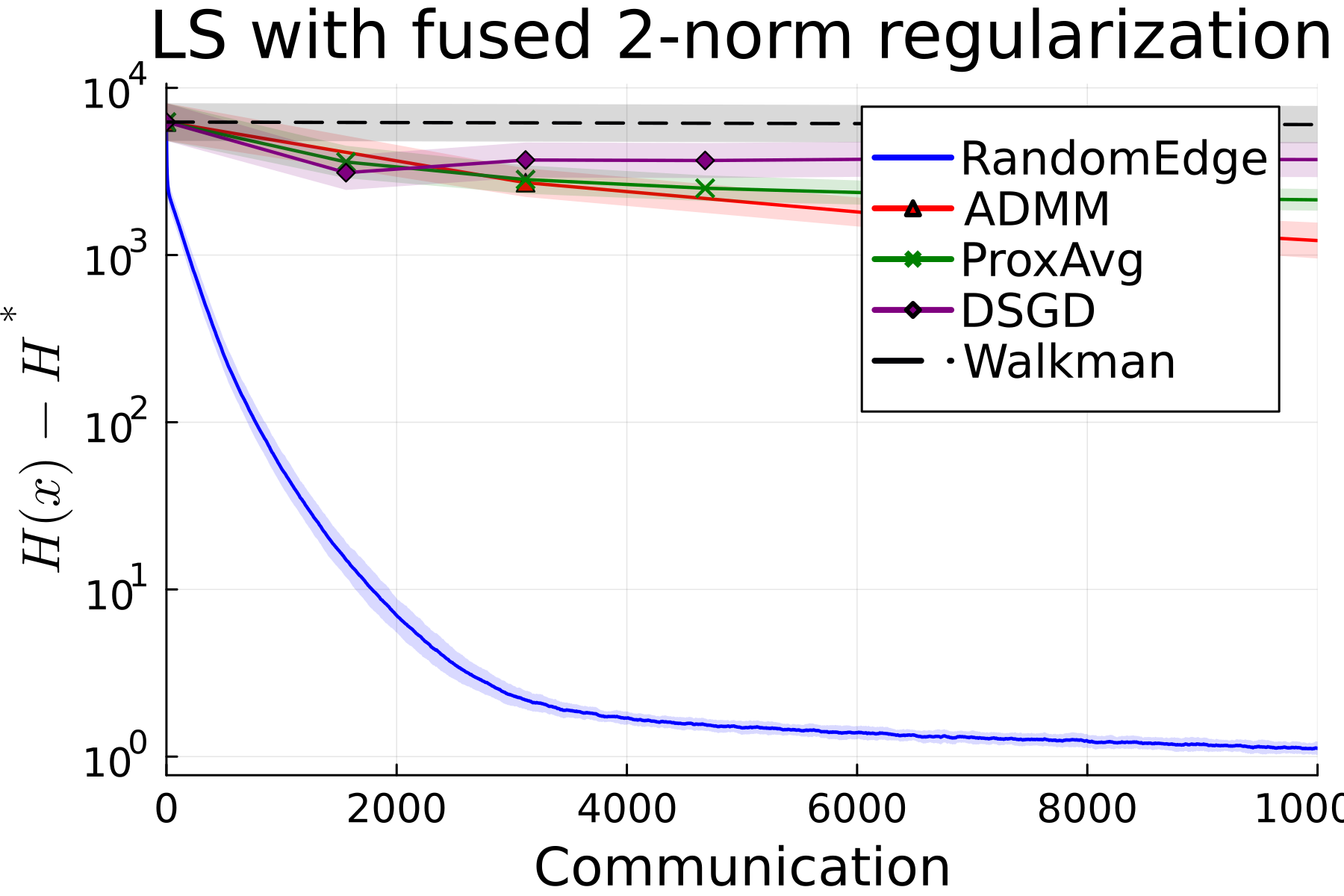}
	}
	\caption{Convergence comparison for the $\ell_2$-regularized least-squares problem~\cref{eq:2-norm-LS} across three network topologies. Error bars show $\pm 1$ standard deviation over 100 experiments.}
	\label{fig:2-norm-LS}
\end{figure*}
\begin{figure*}[t]
	\centering
	\subfloat[5 groups (10, 17, 18, 18, 12 nodes)]{\label{fig:1-norm-LS-5}
		\includegraphics[width=0.3\linewidth]{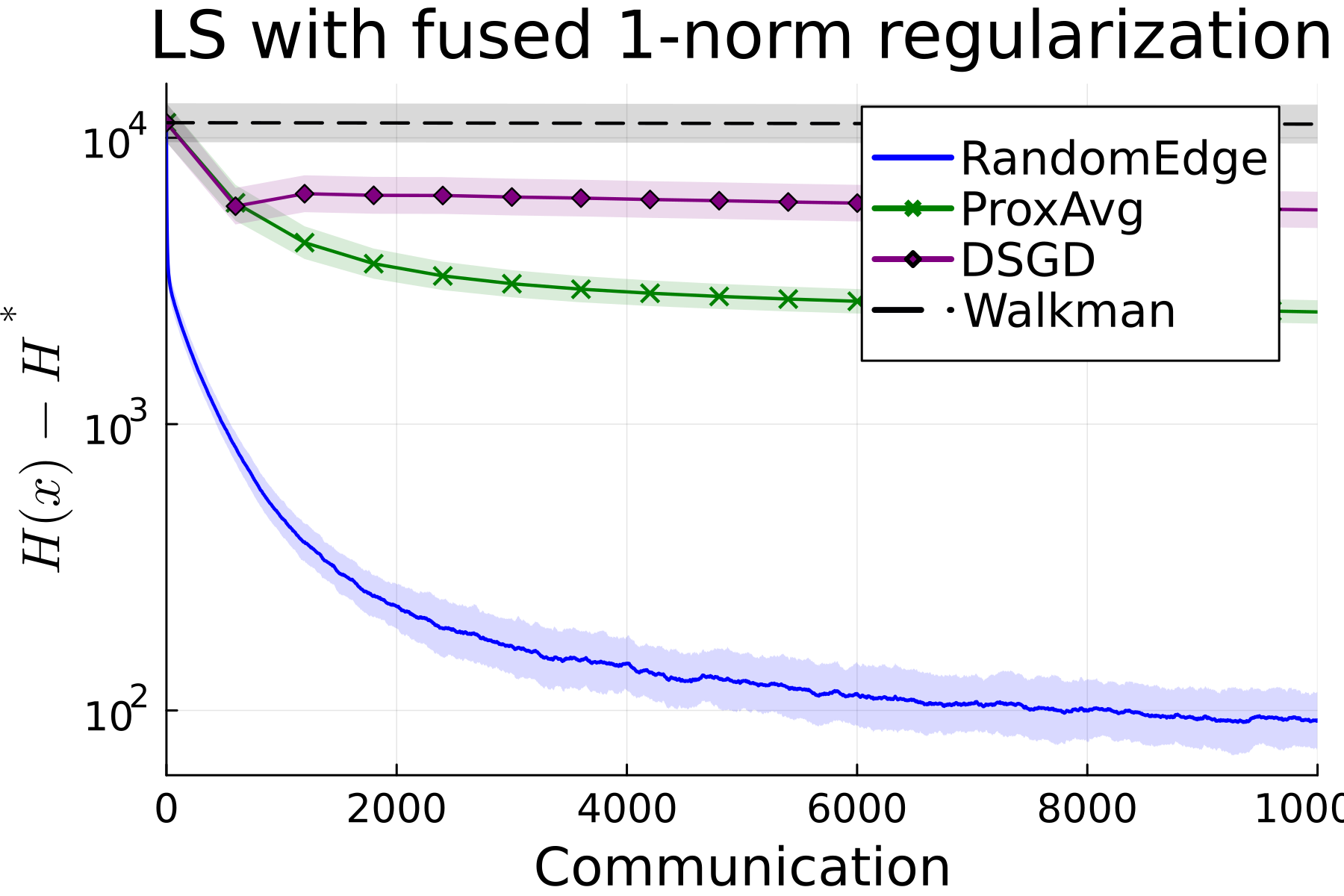}
	}
	\subfloat[1 group (20 nodes)]{\label{fig:1-norm-LS-1}
		\includegraphics[width=0.3\linewidth]{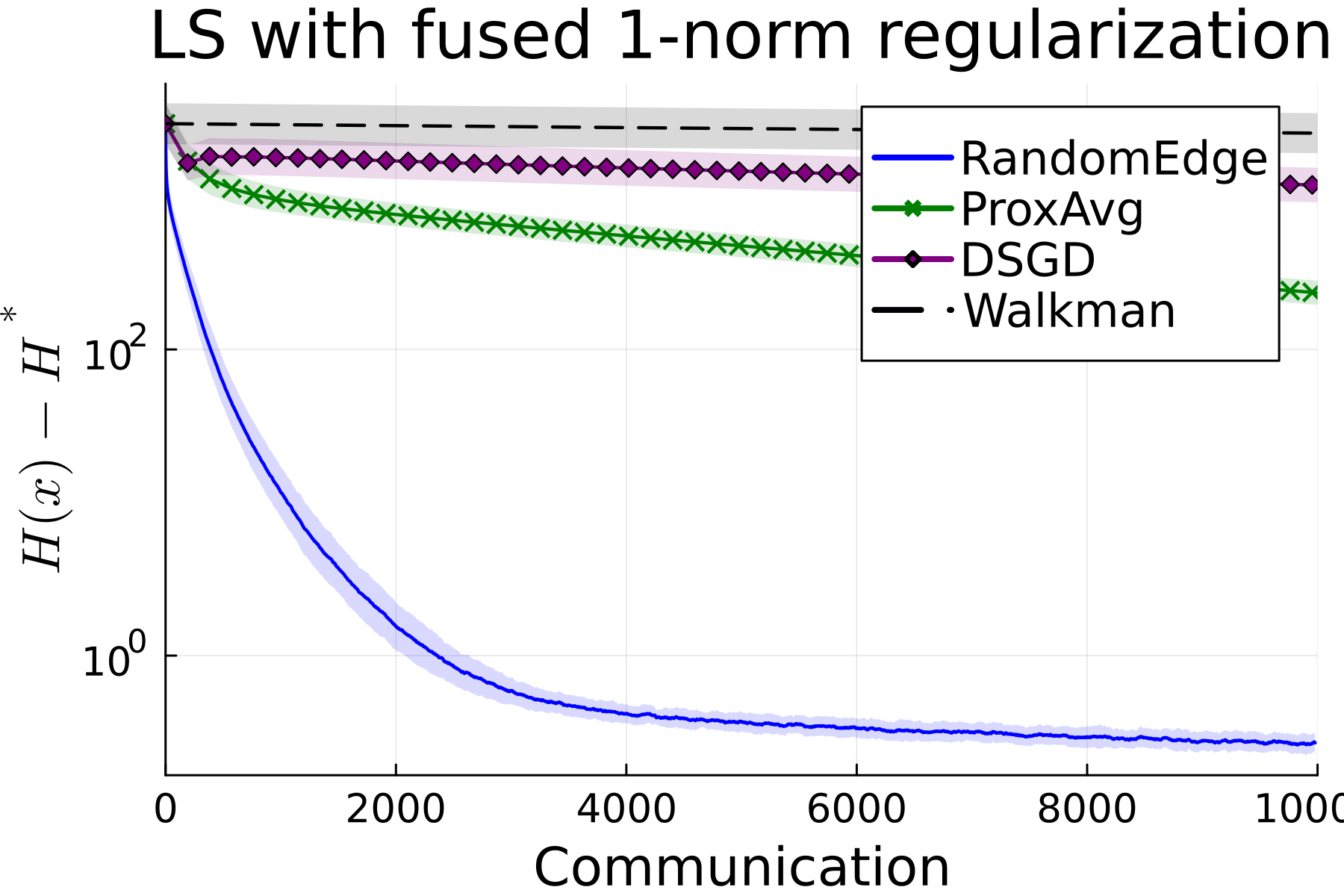}
	}
	\subfloat[1 group (40 nodes, complete graph)]{\label{fig:1-norm-LS-1-fully}
		\includegraphics[width=0.3\linewidth]{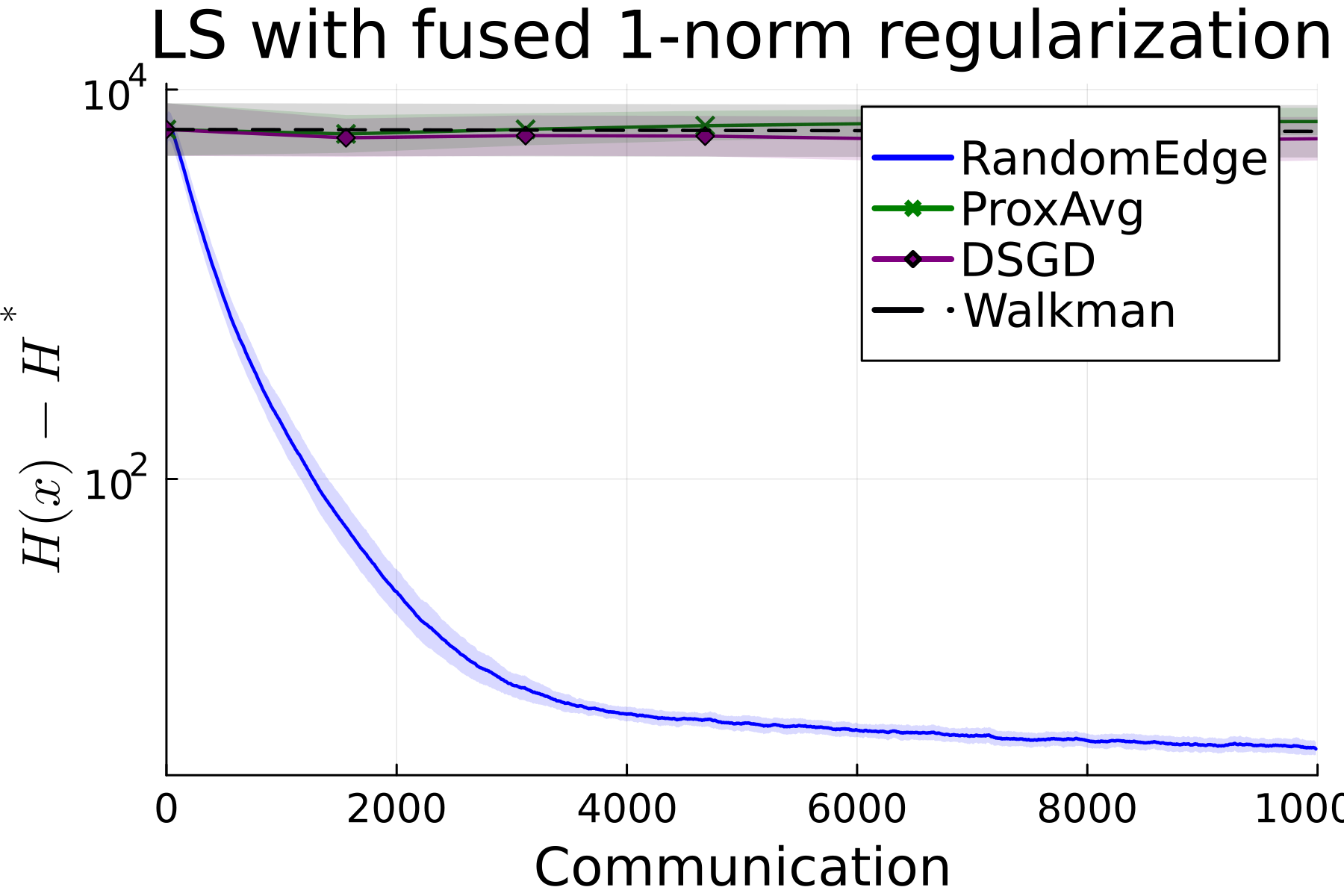}
	}
	\caption{Convergence comparison for the $\ell_1$-regularized least-squares problem~\cref{eq:1-norm-LS} across three network topologies. Error bars show $\pm 1$ standard deviation over 100 experiments.}
	\label{fig:1-norm-LS}
\end{figure*}

\subsection{Real Data: Housing Dataset Application}
\label{sec:real_data}

We demonstrate the practical effectiveness of our methods using the Sacramento housing dataset from~\citet[Section~5.2]{HLB2015}.
(Due to space limitations, we omit here the details of this dataset.)
As with the synthesized experiments, we slightly tune the stepsizes, so that \RandomEdge\ uses \( \alpha^{(t)} = \tfrac{0.003}{\sqrt{t+1}} \) and \( \beta^{(t)} = \alpha^{(t)} \abs{\cE} \); \textsc{ADMM} uses \( \rho = 10^{-4} + \sqrt{\lambda / 2} \); \textsc{ProxAvg} and \textsc{DSGD} use stepsize \( 10^{-2} \); \textsc{Walkman} uses \( \beta = 10000 \).
We fix 50000 communications and perform experiments, the results are shown in \cref{fig:housing}.

\begin{figure}[ht]
  \centering
  \includegraphics[width=0.4\linewidth]{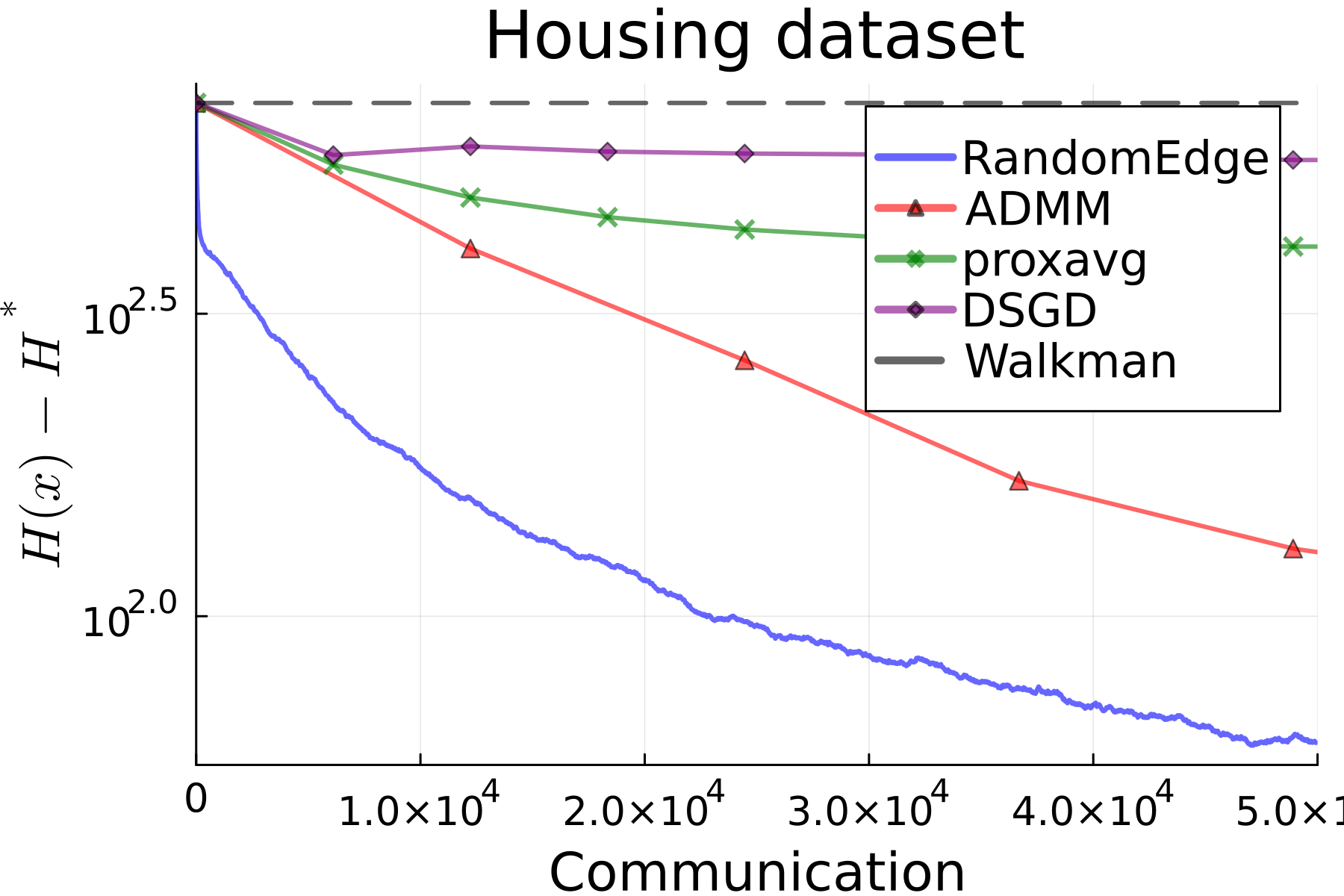}
  \caption{Results for housing dataset.}
  \label{fig:housing}
\end{figure}

\section{Conclusion and Future Work}
\label{sec:conclusion}

We address communication efficiency in decentralized multi-task learning with block-separable and partially separable objectives. \BlockProx\ and its specialization \RandomEdge\ replace $\prox_{\beta G}(z)$ with $\prox_{\beta G_j}(z)$ for a single randomly selected component $G_j$, which reduces communication from global to local coupling structure.

Our analysis establishes iteration complexity of $\tilde{O}(\varepsilon^{-2})$ under bounded subgradients or convexity and smoothnes, and $O(\varepsilon^{-1})$ under strong convexity and smoothness, which demonstrates that communication efficiency preserves convergence guarantees.

Classical approaches require $O(m)$ communications per iteration, while \RandomEdge\ requires only 2, with the reduction scaling with network size.

Experiments on synthetic and real-world datasets confirm these theoretical predictions. The algorithm applies to problem structures where ADMM is inapplicable and consensus-based algorithms alter the optimization objective.

The practical implications extend to federated learning and edge computing, where communication constraints are stringent. Structured multi-task learning with minimal communication overhead addresses real infrastructure limitations, enabling organizations to collaborate on related tasks while preserving data locality and avoiding network congestion.

Several limitations suggest future directions. The requirement that individual components $G_j$ have tractable proximal operators may not hold in general. The dependence of convergence rates on the number of components $m$ suggests potential for improvement in highly structured problems. Asynchronous variants could relax the synchronous communication assumption to accommodate realistic communication patterns. The last-iterate convergence for general convex cases could be interesting.

Our work establishes a principle for communication-efficient distributed optimization: \emph{exploit partial separability through randomized local coordination}. Rather than requiring all coupled variables to coordinate simultaneously, we identify minimal interaction sets and sample from them. This principle extends beyond the specific algorithms presented here and may apply to other optimization contexts where distributed coordination faces communication constraints.

\appendix

\section{Proof of \cref{lem:last_iteration}}\label{app:proof_last_iteration}
For \( k = 0 \) the inequality reduces trivially to \( \alpha_0\theta_0 \leq \alpha_0\theta_0 \).  For \( k \geq 1 \),
\begin{align*}
\alpha_k \theta_k
&= \tfrac{1}{k+1} \sum_{t=0}^k \alpha_t\theta_t + \sum_{h=0}^{k}  \left( \frac{1}{k+1-h} - \frac{1}{k+1} \right) \alpha_h \theta_h - \sum_{h=0}^{k-1} \frac{1}{k+1-h} \alpha_h \theta_h \\
&= \tfrac{1}{k+1} \sum_{t=0}^k \alpha_t\theta_t + \sum_{h=1}^{k} \left( \frac{1}{k+1-h} - \frac{1}{k+1} \right) \alpha_h \theta_h - \sum_{h=0}^{k-1} \frac{1}{k+1-h} \alpha_h \theta_h.
\end{align*}
The key step is recognizing that each difference $1/(k+1-h) - 1/(k+1)$ telescopes:
\begin{align*}
\frac{1}{k+1-h} - \frac{1}{k+1} &= \left(\frac{1}{k+1-h} - \frac{1}{k+2-h}\right) +  \cdots + \left(\frac{1}{k} - \frac{1}{k+1}\right) \\
&= \tsum_{t=k+1-h}^{k} \left( \frac{1}{t} - \frac{1}{t+1} \right).
\end{align*}
Substituting this telescoping identity and continuing the derivation:
\begin{align*}
\alpha_k \theta_k &= \tfrac{1}{k+1} \sum_{t=0}^k \alpha_t\theta_t + \sum_{h=1}^{k} \sum_{t=k+1-h}^{k} \left( \frac{1}{t} - \frac{1}{t+1} \right) \alpha_h \theta_h - \sum_{h=0}^{k-1} \frac{1}{k+1-h} \alpha_h \theta_h \\
&= \tfrac{1}{k+1} \sum_{t=0}^k \alpha_t\theta_t + \sum_{t=1}^{k} \sum_{h=k+1-t}^{k} \left( \frac{1}{t} - \frac{1}{t+1} \right) \alpha_h \theta_h - \sum_{t=1}^{k} \frac{1}{t+1} \alpha_{k-t} \theta_{k-t} \\
&= \tfrac{1}{k+1} \sum_{t=0}^k \alpha_t\theta_t + \sum_{t=1}^{k} \omega_t \sum_{h=k+1-t}^{k} \alpha_h \theta_h - \sum_{t=1}^{k} \frac{1}{t(t+1)} \sum_{h=k+1-t}^{k} \alpha_{k-t} \theta_{k-t} \\
&= \tfrac{1}{k+1} \sum_{t=0}^k \alpha_t\theta_t + \sum_{t=1}^{k} \omega_t \sum_{h=k+1-t}^{k} \left( \alpha_h \theta_h - \alpha_{k-t} \theta_{k-t} \right) \\
&\le \tfrac{1}{k+1} \sum_{t=0}^k \alpha_t\theta_t + \sum_{t=1}^{k} \omega_t \sum_{h=k+1-t}^{k} \alpha_h \left( \theta_h - \theta_{k-t} \right),
\end{align*}
where the second equality holds by swapping the order of summation and changing the indices; the last inequality follows from $\alpha_k$ being non-increasing and $h \ge k-t$ over the summation.

\bibliography{references}
\end{document}